\documentclass[lettersize,journal]{IEEEtran}
\usepackage[T1]{fontenc}    
\usepackage{url}            
\usepackage{booktabs}       
\usepackage{amsfonts}       
\usepackage{nicefrac}       
\usepackage{microtype}      
\usepackage{amsmath}
\usepackage{mathtools}
\usepackage{etoolbox}
\usepackage{cite}
\usepackage{pbalance}
\usepackage{hyperref}

\makeatletter
\patchcmd{\@IEEEyesnumber}
  {\stepcounter}
  {\refstepcounter}
  {}{}
\patchcmd{\@@@IEEEeqnarray}
  {\stepcounter}
  {\refstepcounter}
  {}{}
\patchcmd{\@IEEEeqnarray}{\relax}{\relax\intertext@}{}{}
\patchcmd{\@@IEEEeqnarraycr}
  {\stepcounter{IEEEsubequation}}
  {\refstepcounter{IEEEsubequation}}
  {}{}
\patchcmd{\@@IEEEeqnarraycr}
  {\stepcounter{IEEEsubequation}}
  {\refstepcounter{IEEEsubequation}}
  {}{}
\patchcmd{\@@IEEEeqnarraycr}
  {\stepcounter{IEEEequation}}
  {\refstepcounter{IEEEequation}}
  {}{}
\patchcmd{\@@IEEEeqnarraycr}
  {\stepcounter{IEEEequation}}
  {\refstepcounter{IEEEequation}}
  {}{}
\makeatother

\usepackage{physics}
\usepackage[capitalise]{cleveref}
\usepackage[only,llbracket,rrbracket]{stmaryrd}
\usepackage{amssymb}
\usepackage{amsthm}
\usepackage{mathrsfs}
\usepackage{bm}
\usepackage{bbm}
\usepackage[svgnames,table]{xcolor}
\usepackage{upgreek}
\usepackage{array}
\usepackage{multirow}
\usepackage{multicol}

\usepackage{tikz}
\usetikzlibrary{arrows.meta,calc}

\usepackage{mleftright}
\mleftright

\usepackage{shortcuts}

\newtheorem{lemma}{Lemma}

\newcommand{\bnote}[1]{\textcolor{blue}{\text{#1}}}

\newcolumntype{M}{>{$\displaystyle}l<{$}}
\newcolumntype{N}{>{$\displaystyle}c<{$}}
\newcolumntype{O}{>{$\displaystyle}r<{$}}
\newcolumntype{;}{@{\hspace{0.27778em}}}
\newcolumntype{Z}{@{}}

\newcounter{experiment}
\newcommand{\experiment}{\vspace{0.1in}\noindent\refstepcounter{experiment}%
\textbf{Experiment \theexperiment}. }

\allowdisplaybreaks

\begin{document}

\title{The Spectral Bias of Shallow Neural Network Learning is Shaped by the
Choice of Non-linearity
}

\author{Justin~Sahs${}^1$, Ryan~Pyle${}^1$, Fabio~Anselmi${}^1$, and~Ankit~Patel${}^{1,2,*}$%
\thanks{\hspace{-\parindent}\resizebox{\columnwidth}{!}{\begin{tabular}{@{}l@{}}$^{1}$Department of Neuroscience, Baylor College of Medicine, Houston, TX, 77030, USA\\
$^{2}$Department of Electrical Engineering, Rice University, Houston, TX, 77005, USA\\
$^{*}$corresponding author: \texttt{ankit.patel@rice.edu}\end{tabular}}}}

\maketitle

\begin{abstract}
  Despite classical statistical theory predicting severe overfitting, modern
  massively overparameterized neural networks still generalize well.
  This unexpected property is attributed to the network's so-called implicit
  bias, which describes its propensity to converge to solutions that generalize
  effectively, among the many possible that correctly label the training data.
  The aim of our research is to explore this bias from a new perspective,
  focusing on how non-linear activation functions contribute to shaping it.
  First, we introduce a reparameterization which removes a continuous weight
  rescaling symmetry.
  Second, in the kernel regime, we leverage this reparameterization to
  generalize recent findings that relate shallow Neural Networks to the Radon
  transform, deriving an explicit formula for the implicit bias induced by a
  broad class of activation functions.
  Specifically, by utilizing the connection between the Radon transform and the
  Fourier transform, we interpret the kernel regime's inductive bias as
  minimizing a spectral seminorm that penalizes high-frequency components, in a
  manner dependent on the activation function.
  Finally, in the adaptive regime, we demonstrate the existence of local
  dynamical attractors that facilitate the formation of clusters of hyperplanes
  where the input to a neuron's activation function is zero, yielding alignment
  between many neurons' response functions.
  We confirm these theoretical results with simulations.
  All together, our work provides a deeper understanding of the mechanisms
  underlying the generalization capabilities of overparameterized neural
  networks and its relation with the implicit bias, offering potential pathways
  for designing more efficient and robust models.
\end{abstract}

\section{Introduction}

\IEEEPARstart{T}{he} surprising observation that modern massively
overparameterized Neural Networks (NNs) achieve good generalization, despite
classical statistical predictions suggesting they should heavily
overfit, has led to the study of \emph{inductive bias} (IB) and \emph{implicit
regularization} (IR).
These phenomena posit that the combination of architecture,
initialization, and training algorithm selects global minimizers of the
training loss that also exhibit strong generalization
properties~\cite{Neyshabur2015Search}.

Recent progress towards understanding such IR effects focuses on simple
architectures such as shallow, fully-connected (FC) networks trained with
$\ell_2$ weight decay.
Early work identified the class of functions representable by such networks
with the ReLU activation, first for univariate input~\cite{Savarese2019How},
then extending to multivariate~\cite{Ongie2020Function}.
These works identified a seminorm derived from the Radon transform of the
network's fitted function, such that the representable functions are exactly
those with a finite seminorm.
These results were then extended to the problem of data fitting, leading to the
representer theorem and Banach space characterization presented
in~\cite{Parhi2021Banach, Parhi2023Near-Minimax, Unser2023Ridges}.
This body of work introduces an infinite-dimensional function-space
optimization problem (minimizing the Radon seminorm) whose extreme points are
finite-width NNs that minimize the combined loss of a data-fitting term and
$\ell_2$ weight decay (see \cite[Theorem 12]{Unser2023Ridges} and \cite[Theorem
8]{Parhi2021Banach}).
Some of these results have also been extended to powers of the (leaky) ReLU
activation function~\cite{Parhi2021Banach, Parhi2020Role}, as well as to deep
ReLU networks with rank constraints on the weight
matrices~\cite{Parhi2022What}.

While these analyses represent a significant step toward understanding NN
function space optimization, a full characterization is still lacking.
First, they do not directly analyze real-world learning algorithms such as
gradient descent (GD).
Instead, they consider an entire convex space of solutions spanning various
network widths, whereas GD is applied to a single fixed-width network and
converges to a specific solution--one that may not achieve exactly zero
training error, particularly with early stopping.
Moreover, while these results are mathematically rigorous, they are not
intuitive.
What do functions with low Radon seminorm actually look like?
What structural or qualitative properties do they exhibit?
What is the effect of activation functions outside of the (leaky) power ReLU
family?
Developing a deeper intuition for these aspects remains an open challenge.

Generally, the research in this area can be classified as concerning one of two
training \emph{regimes}: the \emph{kernel} regime, wherein parameter dynamics
are simplified and linear, or the \emph{adaptive} regime, where dynamics retain
full complexity (see~\Cref{sec:alphadegen}).
In~\cite{williams2019gradient}, univariate ReLU networks in the kernel regime
are found to minimize a seminorm based on the second derivative of the network
function.
The work in \cite{Jin2023Implicit} generalizes to the multivariate case,
where the seminorm again involves the Radon transform.
While these results are derived in the context of gradient descent (GD), they
remain limited in scope--either to univariate settings or to multivariate ReLU
networks--and their implications are still challenging to interpret
intuitively.

A separate series of works has employed the so-called ``mean-field'' approach,
which takes an infinite-width limit that preserves adaptivity, yielding an
asymptotic PDE that governs the dynamics of the approximating function
throughout training~\cite{Mei2018Mean, Mei2019Mean-field,
Rotskoff2022Trainability, Chizat2018Global, De_Bortoli2020Quantitative,
Chen2020Dynamical}.
Furthermore, finite-width networks stay close to the asymptotic functions
throughout training~\cite{Rotskoff2022Trainability, Chen2020Dynamical}, and the
finite-width loss landscape does not change much as width
grows~\cite{Mei2018Mean}, independent of input
dimension~\cite{Mei2019Mean-field}.
Additionally, the mean-field PDE takes the form of a so-called \emph{continuity
equation}, as studied in the context of fluid dynamics~\cite{Chizat2018Global}.
This framework is useful for establishing convergence results and approximation
bounds, but, again, is lacking in interpretability: it is unclear how to
translate the mean-field PDE results into meaningful statements about
regularization or generalization.
Once again: What do the trained network functions \emph{look like}?

\paragraph{Main Contributions.}
\begin{itemize}
\item We present a reparameterization of multivariate networks with arbitrary
    activation function~(\Cref{sec:reparam}).
    We show how, in the kernel regime, this reparameterization makes the
    relationship with the Radon transform simpler, and provides a mathematical
    framework for generalizing the results of~\cite{Jin2023Implicit} to a large
    class of activation functions.~(\Cref{sec:kernel})
\item We interpret the induced Radon-space seminorm as a Fourier-space
    penalization composed of two parts: one induced by the shallow FC
    architecture and one induced by the choice of activation
    function~(\Cref{sec:fourier}).
    Thus low-seminorm functions are relatively smooth functions without much
    high frequency content.
\item We leverage this Fourier perspective and show how it enables the design
    of activation functions that impose a desired Fourier-space penalty; we
    further explore the implications of such tailored design for generalization
    and potential challenges posed by the curse of
    dimensionality.~(\Cref{sec:choose,sec:curse})
\item Finally, we consider the adaptive regime (see~\Cref{sec:alphadegen},
    examining the training dynamics and loss landscape structure of ReLU
    networks in the light of the new reparameterization.
    We show how this novel perspective provides intuitive explanations for
    previously-observed phenomena, such as the tendency of network weights to
    ``concentrate in a small number of directions''~\cite{Maennel2018Gradient},
    which manifests in our reparameterization as clustering of parameters
    related to the direction and orientation of hyperplanes associated with
    each neuron, where that neuron's contribution is
    non-linear.~(\Cref{sec:adapt})
    We generalize to general activations in~\Cref{sec:adaptgen}
\end{itemize}

\section{Reparameterization}\label{sec:reparam}

For $D$-dimensional inputs, we write the weight-based NN parameterization of a
shallow ReLU NN with $H$ neurons as
\[ f_{\theta_{\tNN}}\mkern-2mu(\vx)
   = \sum_{i=1}^H v_i\left(\langle\vw_i,\vx\rangle + b_i\right)_+,
\]
where $\theta_{\tNN} \triangleq (\vw_i,b_i,v_i)_{i=1}^H$ with $D$-dimensional
input weights $\vw_i$, scalar biases $b_i$, and scalar output weights $v_i$.
Each term $f_i(\vx) \triangleq v_i\left(\langle\vw_i,\vx\rangle + b_i\right)_+$
is a 2-piece continuous piecewise linear function which is 0 for all $\vx$ on
the ``inactive'' side of the $(D-1)$-plane determined by the equation
$\langle\vw_i,\vx\rangle + b_i = 0$ (referred to as a \emph{breakplane}), and
linear in the distance from that plane on the ``active'' side.
The mapping $(\vw_i,b_i,v_i) \mapsto f_i(\vx)$ is many-to-one.
However, fundamentally, $f_i(\vx)$ of this form belong to a family of functions
uniquely determined by the location and orientation of the breakplane, and the
slope on the active side.
In the setting of univariate ReLU networks, \cite{sahs2022shallow} introduced a
reparameterization that reflects this fact.

Extending this view to the multivariate setting yields a reparameterization
based on the orientation $\vxi_i \triangleq \frac{\vw_i}{\|\vw_i\|_2}$, signed
distance from the origin, $\gamma_i \triangleq \frac{-b_i}{\|\vw_i\|_2}$, and
slope $\mu_i \triangleq v_i\|\vw_i\|_2$.
Dubbing this the \emph{Radon Spline} parameterization $\theta_{\tRS}$ based on
the relationship with the Radon transform discussed below
in~\Cref{sec:radonintro}, we can write
\begin{equation}\label{eq:reparamrelu}
  f_{\theta_{\tRS}}\mkern-2mu(\vx)
   = \sum_{i=1}^H \mu_i \left( \langle \vxi_i, \vx \rangle - \gamma_i\right)_+.
\end{equation}

\subsection[Training regimes and alpha-degeneracy]{Training regimes and $\alpha$-degeneracy}\label{sec:alphadegen}
Because $(\cdot)_+$ is 1-homogeneous, the mapping from $\theta_{\tNN}$ to
$\theta_{\tRS}$ is many-to-one: the underlying function, and hence
$\theta_{\tRS}$, is invariant under the mapping $(\vw_i,v_i,b_i) \mapsto
(\alpha_i\vw_i,\frac{v_i}{\alpha_i},\alpha_i b_i)$.
We call this the \emph{$\alpha$-degeneracy} or \emph{$\alpha$-symmetry}.
Adding an additional parameter, $\omega_i \triangleq \|\vw_i\|_2$ yields the
$\theta_{\tRS,\omega}$ parameterization, which is no longer invariant to the
$\alpha$-symmetry, making it one-to-one with $\theta_{\tNN}$.
Although the underlying function is invariant under the $\alpha$-symmetry, and
hence $f_{\theta_{\tRS,\omega}}\mkern-2mu(\vx)$ does not depend on
$(\omega_i)_i$, the training dynamics under gradient descent are affected by
the $\alpha$ mapping (as first studied in~\cite{Chizat2019Lazy}).
We can measure the effect of $\alpha$ on training by the derived statistic
$\delta_i \triangleq v_i^2 - \|\vw_i\|_2^2 - b_i^2 = \mu_i^2/\omega_i^2 -
\left(\gamma_i^2+1\right)\omega_i^2$, which is generalized from the
1-dimensional version found in~\cite{williams2019gradient}.
$\delta_i$ is not invariant under the $\alpha$-symmetry, but \emph{is}
invariant under gradient descent, i.e.\ they depend only on the initial values
of the parameters $\theta_{\tRS,\omega}$.

As $\delta_i \to -\infty$ (e.g.\ under a $\alpha_i\to\infty$ transformation),
breakplanes stop changing, so that only the delta-slopes change.
This effectively transforms our learning problem into learning a set of weights
for a fixed basis set; we call this the \emph{kernel
regime}~\cite{Woodworth2020Kernel}.
In other words, in the kernel regime, only $(\mu_i)_{i=1}^H$ is trained, with
$(\vxi_i,\gamma_i,\omega_i)_{i=1}^H$ constant.
This regime has also been studied under the names \emph{linear
regime}~\cite{Luo2021Phase} and \emph{lazy training}~\cite{Chizat2019Lazy}.

Conversely, as $\delta_i \to \infty$ (i.e.\ $\alpha_i\to 0$), breakplane motion
becomes an integral part of training.
We call this the \emph{adaptive regime}~\cite{williams2019gradient}; it has
also been studied under the name \emph{rich
regime}~\cite{Woodworth2020Kernel, Li2022What} and \emph{critical
regime}~\cite{Luo2021Phase}.

Recently~\cite{Kunin2024Get} has shed more light on this phenomenon by
considering per-layer learning weights, $\eta_1$ (governing the learning rate
of $\vw_i$ and $b_i$) and $\eta_2$ (governing the rate of $v_i$).
Under this approach, we redefine $\delta_i \triangleq \eta_1 v_i^2 -
\eta_2\|\vw_i\|_2^2 - \eta_2 b_i^2 = \eta_1\mu_i^2/\omega_i^2 - \eta_2
\left(\gamma_i^2+1\right)\omega_i^2$.
Then, these new weights can be tuned to select along the kernel-adaptive
spectrum, independently of the scale of the initialization.

\subsection{Arbitrary Activation Functions}

We now generalize this parameterization to arbitrary activation functions
$\phi(\cdot)$.
The parameters $(\mu_i,\vxi_i,\gamma_i)$ are kept, but their meaning is
generalized: $\mu_i$ becomes a scale parameter, rather than a slope,
$(\vxi_i,\gamma_i)$ now parameterize the location and orientation of a
\emph{zero-plane} where the input to the activation crosses zero (from negative
to positive as you move from inactive to active side), rather than a
breakplane.
Finally, the underlying function is no longer invariant to the parameter
$\omega_i$, which now parameterizes the horizontal rescaling of the activation,
$\phi_{\omega_i}\!(z) \triangleq \frac{1}{\omega_i}\phi(\omega_i z)$.
G+eneralizing \Cref{eq:reparamrelu}, we get
\begin{equation}\label{eq:reparam}
  f_{\theta_{\tRS,\omega}}\!(\vx)
   = \sum_{i=1}^H \mu_i \phi_{\omega_i}\left(\langle\vxi_i,\vx\rangle - \gamma_i\right).
\end{equation}
When $\phi(\cdot)$ is 1-homogeneous (as is the case with the ReLU activation),
$\phi_{\omega_i}\!(\cdot) = \phi(\cdot)$, so that $\omega_i$ becomes a
redundant parameter.
However, in the adaptive regime, $\omega_i$ can still affect parameter
dynamics, even though $f_{\theta_{\tRS}}\mkern-2mu(\vx)$ is unaffected.

\subsection{Relationship with the Radon Transform}\label{sec:radonintro}

In cases where the activation function $\phi$ is 1-homogeneous or that
$\omega_i = 1$ for all $i$ at all times, we can rewrite the sum in
\Cref{eq:reparam} as an integral:
\begin{IEEEeqnarray*}{r;l}
  f_{\theta_{\tRS}(t)}\mkern-1mu(\vx)
    &\triangleq \sum_{i=1}^H\mu_i\phi(\langle\vxi_i,\vx\rangle-\gamma_i)\\
    &\triangleq
       \int\limits_{\mathclap{\SS^{D-1}\times\RR\times\RR}}
         \mu\phi(\langle\vxi,\vx\rangle-\gamma)
       \dd{\eta_t(\vxi,\gamma,\mu)}\\
\shortintertext{where}
  \eta_t(\vxi,\gamma,\mu)
    &\triangleq \sum_{i=1}^H \delta_{(\vxi_i,\gamma_i,\mu_i)}
\end{IEEEeqnarray*}
is the (un-normalized) empiric distribution of parameters at time $t$.
By letting $\eta_t(\ldots)$ be an arbitrary measure, we can represent infinite
width networks.
Below, we use $\eta_t(\ldots)$ to also denote the density of $\eta_t(\ldots)$;
in the case that $\eta_t(\ldots)$ has atoms, we understand this density to be a
Schwartz distribution.
Rearranging in terms of conditional and marginal densities, we get
\begin{IEEEeqnarray*}{r;l}
  \mathrlap{f_{\theta_{\tRS}(t)}\mkern-1mu(\vx)}
  \hspace{8mm}
  \\
    &= \int\limits_{\mathclap{\SS^{D-1}\times\RR}}
         \left(
           \int_\RR
             \mu
             \eta_t(\mu|\vxi,\gamma)
           \dd{\mu}
         \right)
         \phi(\langle \vxi, \vx\rangle - \gamma)
         \eta_t(\vxi,\gamma)
       \dd{\vxi}\dd{\gamma}\\
    &\triangleq
       \int\limits_{\mathclap{\SS^{D-1}\times\RR}}
         c_t(\vxi,\gamma)
         \phi(\langle \vxi, \vx\rangle - \gamma)
         \eta_t(\vxi,\gamma)
       \dd{\vxi}\dd{\gamma}\\
    &= \int\limits_{\mathclap{\SS^{D-1}}}
         \left(
           \phi
           \ast_\gamma
           c_t(\vxi,\cdot)\eta_t(\vxi,\cdot)
         \right)
         (\langle \vxi, \vx\rangle)
       \dd{\vxi}\\
\end{IEEEeqnarray*}
where $\ast_\gamma$ denotes convolution in the $\gamma$ variable.
Finally, we can rewrite the last equality as
\begin{equation}\label{eq:nndualradon}
  f_{\theta_{\tRS}(t)}\mkern-1mu(\vx)
  = \R^*
    \left\{
      \left(
        \phi
        \ast_\gamma
        c_t(\vxi,\cdot)\eta_t(\vxi,\cdot)
      \right)
      (\gamma)
    \right\}
    (\vx)
\end{equation}
where $\R^*\{\cdot\}(\vx)$ denotes the Dual Radon transform
\[ \R^*\{\varphi\}(\vx)
   \triangleq
     \int\limits_{\mathclap{\SS^{D-1}}} \varphi(\vxi,\langle\vxi,\vx\rangle) \dd{\vxi}
\]
which takes a function on hyperplanes $\varphi(\cdot,\cdot) :
\SS^{D-1}\times\RR \to \RR$ and converts it to a function on points,
$\R^*\{\varphi\} : \RR^D\to\RR$ by integrating over all hyperplanes that pass
through $\vx$.
As the name implies, the Dual Radon transform is dual to the Radon transform of
a function $f(\cdot) : \RR^D\to\RR$, given by
\[
    \R\{f\}(\vxi,\gamma) = \int\limits_{\mathclap{\langle\vxi,\vx\rangle = \gamma}} f(\vx)\dd{\vx},
\]
which integrates over all $\vx$ on the hyperplane defined by $(\vxi,\gamma)$
(see, e.g.~\cite{Helgason2011Integral, Helgason1980Radon}).

An intuitive understanding for the Radon and dual Radon transforms comes from
the field of medical imaging~\cite{Kuchment2013Radon, Beatty2012Radon}.
In (the basic form of) Computed Tomography (CT), a linear array of parallel
X-ray beams are shot through a patient, and a linear array of sensors records
the resulting intensities on the other side.
Then, the source and sensor arrays are rotated around the patient, producing a
large number of beams with various orientations and offsets.
Each beam effectively computes the integral of the density of the patient along
a line, one for each orientation and offset pair $(\vxi,\gamma)$.
In other words, the CT scanner is computing the Radon transform
$\R\{f\}(\vxi,\gamma)$ of the density of the 2D slice of the patient, yielding
a so-called sinogram.
The original density function can be recovered from this data by using the
inversion formula for the Radon transform:
\[ f(\vx)
   = \kappa_D
     \R^*\left\{\left(-\pdv[2]{\gamma}\right)^{\!\!\frac{D-1}{2}}\!\R\{f\}\right\}(\vx)
\]
where the fractional power $\left(-\pdv[2]{\gamma}\right)^{\!\!\frac{D-1}{2}}$
is typically defined via its Fourier transform, and $\kappa_D$ is a constant
that only depends on $D$.
A similar formula goes in the other direction:
\[ \varphi(\vxi,\gamma)
    = \kappa_D \R\left\{\left(-\nabla^2\right)^{\!\frac{D-1}{2}} \R^*\{\varphi\}\right\}(\vxi,\gamma)
\]
where the $\left(-\nabla^2\right)^{\!\frac{D-1}{2}}$ term is called the
fractional Laplacian.
In these formulae, the fractional differential operators act as low-pass
filters that are necessary to avoid ``overcounting'' points far from the
origin.

\section{Kernel Regime}\label{sec:kernel}

In the kernel regime, only the parameters $\vmu \triangleq (\mu_i)_{i=1}^H$ are
trained.
Because the other parameters are fixed, this turns the network into just a
linear model with fixed features
$\Phi\left(\vx;(\vxi_i,\gamma_i,\omega_i)_{i=1}^H\right) \triangleq
(\phi(\langle\vxi_i,\vx\rangle-\gamma_i))_{i=1}^{H}$.
We can then write the network output as $f_{\theta_{\tRS,\omega}}\!(\vx) =
\Phi\left(\vx;(\vxi_i,\gamma_i,\omega_i)_{i=1}^H\right) \vmu$.
Provided the model can reach zero error, this leads to the solution
\begin{equation}\label{eq:muopt}
  \begin{IEEEeqnarrayboxm}[][c]{c}
    \hat{\vmu}
      = \argmin_{\vmu} \|\vmu-\vmu_0\|_2^2\\
    \text{ s.t.\ }\ y_n = \Phi\left(\vx_n;(\vxi_i,\gamma_i,\omega_i)_{i=1}^H\right) \vmu\ \forall n.
  \end{IEEEeqnarrayboxm}
\end{equation}
Thus, we have a $\ell_2$-regularized feature linear regression with fixed
features $\Phi(\ldots)$.

For simplicity of exposition, let us assume $\vmu_0=\vec{0}$ so that we
minimize $\|\vmu\|_2^2$, and have $f_{\theta_{\tRS}(0)}\mkern-1mu(\vx) \equiv 0$.
Additionally, as in \Cref{sec:radonintro}, we assume that either
$\eta_0(\omega) = \delta_{1}$, i.e.\ that $\|\vw\|_2=1$ at initialization with
probability 1, or that $\phi(\cdot)$ is 1-homogeneous.
Then, as before, we can represent a finite sum as an integral, writing
\begin{equation}\label{eq:ctsquared}
  \|\vmu_t\|_2^2
    = \int\limits_{\mathclap{\SS^{D-1}\times\RR\times\RR}}
        \mu^2
      \dd{\eta_t}\!(\vxi,\gamma,\mu)
    = \int\limits_{\mathclap{\SS^{D-1}\times\RR}}
        c_t^2(\vxi,\gamma)
        \eta_t(\vxi,\gamma)
      \dd{\vxi}\dd{\gamma}
\end{equation}

Then, let $\L_\phi$ be the linear operator defined by convolution with the
activation function $\phi$, that is $\L_\phi \varphi \triangleq \phi
\ast_\gamma \varphi$ for any function $\varphi : \RR\to\RR$ such that the
convolution converges.
Then, we can extend $\L_\phi$ to an operator $\L_{\phi,\vxi}$ by
$\left(\L_{\phi,\vxi} g\right)(\vx) \triangleq \L_\phi g(\vx +\gamma\vxi)$,
i.e.\ by applying the convolution ``in the direction'' $\vxi$.
Then, if $\L_\phi$ has a unique inverse\footnote{As we see in
\Cref{tab:activation}, many activation functions yield a well-defined
$\L_\phi^{-1}$.
Without our assumption that $\omega=1$, \Cref{eq:nndualradon} would have an
extra integral $\dd{\omega}$, which would also need to be inverted, but we
would only expect a unique inverse in special circumstances.}, we can use the
inversion formula for the Dual Radon transform to solve \Cref{eq:nndualradon}
for $c_t(\vxi,\gamma)$, which yields
\begin{IEEEeqnarray*}{r;l}
  c_t(\vxi,\gamma)
    &= \frac{\kappa_D}{\eta_0(\vxi,\gamma)}
       \L_\phi^{-1}
       \R\left\{
         \left(-\nabla^2\right)^{\!\frac{D-1}{2}}
         f_{\theta_{\tRS}(t)}
       \right\}(\vxi,\gamma)\\
    &= \frac{\kappa_D}{\eta_0(\vxi,\gamma)}
       \R\left\{
         \left(-\nabla^2\right)^{\!\frac{D-1}{2}}
         \L_{\phi,\vxi}^{-1}
         f_{\theta_{\tRS}(t)}
       \right\}(\vxi,\gamma)\\
    &\triangleq
       \frac{1}{\eta_0(\vxi,\gamma)}
       \left(\R^*\right)^{-1}
       \left\{
         \L_{\phi,\vxi}^{-1}
         f_{\theta_{\tRS}(t)}
       \right\}(\vxi,\gamma)
\end{IEEEeqnarray*}
We can then substitute this expression for $c_t(\vxi,\gamma)$ into
\Cref{eq:ctsquared} and combine with \Cref{eq:muopt}, yielding
\begin{equation}
\begin{IEEEeqnarraybox}[][c]{c}
  f_{\widehat{\theta}_{\tRS}}
    = \argmin_{f \in \mathscr{F}_{\!\phi}}
        \int\limits_{\mathclap{\SS^{D-1}\times\RR}}
          \frac
            {\left(
               \left(\R^*\right)^{-1}
               \left\{
                 \L^{-1}_{\phi,\vxi} f
               \right\}(\vxi,\gamma)
             \right)^2
            }
            {\eta_0(\vxi,\gamma)}
        \dd{\vxi}\dd{\gamma}
        \\
  \text{ s.t.\ }
  f(\vx_n) = y_n\ \forall n,
\end{IEEEeqnarraybox}\label{eq:kernelreg}
\end{equation}
where the minimization is over the space $\mathscr{F}_{\!\phi}$ of functions
such that the integral being minimized is finite.

If we consider the special case $\phi=(\cdot)_+$, we can see that
\[((\cdot)_+ \ast \varphi)(\gamma)
  = \int_{-\infty}^\gamma (\gamma-t) \varphi(t) \dd{t}\!,
\]
which has the form of the Cauchy formula for repeated integration, i.e.\
$\L_\phi \varphi = \iint \varphi(t) \dd[2]{t}$.
From this, it is clear that $\L_\phi$ is inverted by twice differentiating,
$\L_\phi^{-1} \varphi = \dv[2]{\varphi}{\gamma}$.
The Radon transform is said have an ``intertwining'' property that
$\dv[2]{}{\gamma} \R\{f\} = \R\left\{\nabla^2 f\right\}$, so instead of using
$\L_{\phi,\vxi}^{-1} = \partial^2_{\vxi}$ (the second derivative in the
direction $\vxi$), we can use this to specialize \Cref{eq:kernelreg} and
reproduce \cite[Theorem 6]{Jin2023Implicit}:
\begin{IEEEeqnarray*}{c}
  f_{\widehat{\theta}_{\tRS}}
    =
    \argmin_{f \in \mathscr{F}_{\!\phi}}
      \int\limits_{\mathclap{\SS^{D-1}\times\RR}}
        \frac
          {\left(
             \left(\R^*\right)^{-1}
             \left\{
               \nabla^2 f
             \right\}(\vxi,\gamma)
           \right)^2
          }
          {\eta_0(\vxi,\gamma)}
      \dd{\vxi}\dd{\gamma}\\
  \text{ s.t.\ }
  f(\vx_n) = y_n\ \forall n.\\
  \phi=\text{ReLU}
\end{IEEEeqnarray*}
This result is formalized in \cite{Jin2023Implicit}, including technical
requirements for the infinite width limit to converge, and rates of
convergence.

We call the numerator $\left( \left(\R^*\right)^{-1} \left\{
\L^{-1}_{\phi,\vxi} f \right\}(\vxi,\gamma) \right)^2$ in \Cref{eq:kernelreg}
the \emph{representational cost} of $f(\vx)$ along the hyperplane
$\langle\vxi,\vx\rangle=\gamma$, which is a measure of the ``local difficulty''
of implementing $f(\vx)$, where ``local'' means ``confined to the hyperplane''.
In the case of the ReLU activation, this corresponds to the integral of
Laplacian curvature along the hyperplane.
Then, the denominator $\eta_0(\cdot,\cdot)$ serves as a hyperplane weighting
factor which increases the importance of \emph{low density} regions; such
regions are therefore even more regularized (e.g.\ must have very low
curvature).
The intuition here is that a region of $(\vxi,\gamma)$-space with low density
corresponds to a region of $\vx$-space with few or no zero-planes; with no
zero-planes, the network cannot ``afford'' any representational cost in that
region.
In the kernel regime, the network cannot move zero-planes, so it must
necessarily find a solution with low representational cost in that
region (assuming such a solution exists).
In the ReLU activation example, a region with no zero-planes is necessarily
affine, so the network cannot implement any curvature in such a region.

If we let $\psi(\vxi,\gamma)$ be a measure on $\supp \eta_0$ with density
$\frac{1}{\eta_0(\vxi,\gamma)}$, we can write the objective of
\Cref{eq:kernelreg} as
\[ \int\limits_{\mathclap{\supp \eta_0}}
     \left(
       \left(\R^*\right)^{-1}
       \left\{
         \L^{-1}_{\phi,\vxi} f
       \right\}(\vxi,\gamma)
     \right)^2
   \dd{\psi(\vxi,\gamma)}.
\]
From this representation, we can see that this is the square of the
$L^2(\supp \eta_0,\psi)$-norm of $\left(\R^*\right)^{-1} \left\{
\L^{-1}_{\phi,\vxi} f \right\}$.
Because $\L^{-1}_{\phi,\vxi}$ and $\left(\R^*\right)^{-1}$ are linear in $f$,
the composition defines a seminorm\footnote{Both $\L^{-1}_{\phi,\vxi}$ and
$\left(\R^*\right)^{-1}$ have non-trivial null spaces, so we only get a
positive semi-definite functional, hence this is a seminorm instead of a
norm.}:
\[ \|f\|_{\R,\phi,\eta_0}
     \triangleq
     \left\|
       \left(\R^*\right)^{-1} \left\{ \L^{-1}_{\phi,\vxi} f \right\}
     \right\|_{L^2(\supp \eta_0,\psi)}
\]
We refer to this as the ``Radon seminorm'' of $f$.

\subsection{Fourier Interpretation}\label{sec:fourier}
The central slice theorem (see e.g.~\cite[p. 32]{Kuchment2013Radon}
or~\cite[p. 4]{Helgason2011Integral}) relates the ($D$-dimensional) Radon
transform to the (1-dimensional and $D$-dimensional) Fourier transform via
\[\F_\gamma\left[\R\{g\}\right](\vxi,\vartheta) = \F_D[g](\vartheta\vxi). \]
Using this result, we can move the squared term of \Cref{eq:kernelreg} to Fourier
space, giving
\begin{equation}\label{eq:fourierreg}
  \begin{IEEEeqnarrayboxm}[][c]{l}
    \|f\|_{\R,\phi,\eta_0}^2\\
    \hspace{5mm}
    = \int\limits_{\mathclap{\SS^{D-1}\times\RR}}
         \frac{\kappa_D^2}{\eta_0(\vxi,\gamma)}
         \left(
           \F_\gamma^{-1}
           \left[
             \frac{|\vartheta|^{D-1}}{\F_\gamma[\phi](\vartheta)}
             \F_D[f](\vartheta\vxi)
           \right](\gamma)
         \right)^2
       \dd{\vxi}\dd{\gamma}\!.
  \end{IEEEeqnarrayboxm}
\end{equation}
From this, we see that fractional Laplacian from the Radon inversion formula
and the convolutional inverse of the activation function act as high-pass
filters, so that the overall regularization is to dampen high frequencies.
This frequency-based regularization is modulated by the $1/\eta_0(\vxi,\gamma)$
term such that regions of low density are \emph{more} regularized.

This connection with Fourier analysis should not be too surprising: some hints
at such a relationship have existed as far back as Barron's 1993
paper~\cite{Barron1993Universal}, where the study of superpositions of sigmoid
functions is restricted to functions whose Fourier transform have finite first
moment, indicating that functions with ``too much'' high frequency content are
hard to approximate with NNs.
Additionally, more recent works have observed empirically (and, in the case of
shallow kernel-regime learning, with some theoretical support) that (deep) NNs
fit lower frequencies first~\cite{Zhang2019Explicitizing, Xu2019Training,
Xu2020Frequency, Rahaman2019Spectral}.

To fully understand and interpret \Cref{eq:fourierreg}, we consider its
component pieces, starting with the ``innermost'' term
$\F_D[f](\vartheta\vxi)$; this corresponds to a minimal objective of
\begin{IEEEeqnarray*}{r;l}
  \cO_1(f)
    &= \int\limits_{\mathclap{\SS^{D-1}\times\RR}}
         \left(\F_\gamma^{-1}\left[\F_D[f](\vartheta\vxi)\right](\gamma)\right)^2
       \dd{\vxi}\dd{\gamma}\\
    &= \int\limits_{\mathclap{\SS^{D-1}\times\RR}}
         \left|\F_D[f](\vartheta\vxi)\right|^2
       \dd{\vxi}\dd{\vartheta},\\
\shortintertext{%
where we have used Plancherel's theorem to evaluate the integral in frequency
space.
This objective minimizes the $L^2$ norm of the Fourier transform, \emph{in
non-Euclidean ``radial'' coordinates}.
Reparameterizing into Euclidean coordinates gives}
    &=2\int\limits_{\mathclap{\RR^D}}
         \frac{1}{k^{D-1}}
         \left|\F_D[f](\vk)\right|^2
       \dd{\vk}\\
\shortintertext{%
where $k \triangleq \|\vk\|_2$ (note that the $1/k^{D-1}$ term is
\emph{outside} the squared modulus); the factor of 2 comes from the fact that
the original integral ``double-counts'' because
$(-\vartheta)(-\vxi)=\vartheta\vxi$.
Next, we consider the term $|\vartheta|^{D-1}$ from \Cref{eq:fourierreg}, which
corresponds to taking the fractional Laplacian
$\left(-\nabla^2\right)^{(D-1)/2}$ of $f$, and comes from the inversion formula
for the dual Radon transform.
In the medical imaging literature \cite{Kuchment2013Radon}, this high-pass
filter is referred to as a ``deblurring'' operation: applying the Radon then
dual Radon transforms without it results in a blurred version of the original
input function.
Re-introducing this term gives the new intermediate objective
}
  \cO_2(f)
    &=2\int\limits_{\mathclap{\RR^D}}
         \frac{1}{k^{D-1}}
         \left|k^{D-1}\F_D[f](\vk)\right|^2
       \dd{\vk}\\
    &=2\int\limits_{\mathclap{\RR^D}}
         \left|
           k^{(D-1)/2}
           \F_D[f](\vk)
         \right|^2
       \dd{\vk}
\shortintertext{%
Thus, we are now \textit{penalizing higher frequencies more than lower ones}
via the factor $k^{(D-1)/2}$.
Next, we re-introduce the activation function term.
Because the original form is a function of $\vartheta$, the ``double-counting''
that lead to the 2 out front could be broken.
However, because $\phi(\cdot)$ is purely real, $\F_\gamma[\phi](-\vartheta) =
\overline{\F_\gamma[\phi](\vartheta)}$ and we have
$|\F_\gamma[\phi](-\vartheta)|=|\F_\gamma[\phi](\vartheta)|$, so we are still
double counting, yielding}
  \IEEEyesnumber
  \cO_3(f)
    &=2\int\limits_{\mathclap{\RR^D}}
         \left|
           \frac{k^{(D-1)/2}}{\F_\gamma[\phi](k)}
           \F_D[f](\vk)
         \right|^2
       \dd{\vk}\\
\shortintertext{%
This adds another weight based on the magnitude of $\vk$; for typical
activation functions, this gives high weight to large frequency magnitudes.
This objective corresponds to a hypothetical uniform density of zero-planes
throughout all of $\C^{D-1}$.
We can expand the modulus-squaring as
}
    \IEEEyesnumber\label{eq:sepsquares}
    &=2\int\limits_{\mathclap{\RR^D}}
         \frac{k^{D-1}}{\left|\F_\gamma[\phi](k)\right|^2}
         \left|
           \F_D[f](\vk)
         \right|^2
       \dd{\vk}\\
    &\triangleq
      2\int\limits_{\mathclap{\RR^D}}
         \rho_{\R,\phi}(k)
         \left|
           \F_D[f](\vk)
         \right|^2
       \dd{\vk}\\
\shortintertext{%
where we call $\rho_{\R,\phi}(k) \triangleq \rho_{\R}(k) \rho_{\phi}(k)$ the
\emph{spectral penalty} induced by the architecture and activation function
$\phi$; we call $\rho_{\R}(k) \triangleq k^{D-1}$ and $\rho_{\phi}(k)
\triangleq 1/\left|\F_\gamma[\phi](k)\right|^2$ the \emph{factors} of the
spectral penalty corresponding to the architecture and activation,
respectively.
Examples of $\rho_\phi(k)$ for various $\phi(z)$ are shown
in~\Cref{tab:activation,fig:activation}).
}
\end{IEEEeqnarray*}

To relate these back to \Cref{eq:fourierreg} more explicitly, we note that
$\cO_1(f)$ corresponds to $\|f\|_{\R,\phi,\eta_0}^2$ for $\phi(\cdot)$ such
that $\F_\gamma[\phi](\vartheta) = |\vartheta|^{D-1}$ and an improper
``density'' $\eta_0(\cdot,\cdot)\equiv 1$.
$\cO_2(f)$ is equivalent to $\|f\|_{\R,\phi,\eta_0}^2$ for $\phi(\cdot)$ such
that $\F_\gamma[\phi](\vartheta) = 1$, i.e. $\phi(\cdot) = \delta(\cdot)$, and
$\eta_0(\cdot,\cdot)\equiv 1$.
$\cO_3(f)$ allows for an arbitrary activation function, but retains the
improper zero-plane density.

Ideally, we would use the convolution theorem then Plancherel's theorem to
re-introduce the $1/\eta_0(\vxi,\gamma)$ term and have a form of
\Cref{eq:fourierreg} entirely in Fourier space.
Unfortunately because $\eta_0(\vxi,\gamma)$ is a density,
$\lim_{\gamma\to\infty}1/\eta_0(\vxi,\gamma)=\infty$, so
$1/\eta_0(\vxi,\gamma)$ is not in $L^2$, and does not have a Fourier transform.
In other words, the zero-plane density term cannot be directly interpreted in
Fourier space.

\begin{table*}
\centering
\resizebox{0.84\textwidth}{!}{
\renewcommand{\arraystretch}{1.2}
\begin{tabular}{r|M|M|Ml}
  Name
    & \text{Activation Function}
    & \text{Filter}
    & \text{Spectral Penalty}
    \\
    & \phi(z)
    & \F[\phi](k)^{-1}
    & \rho_\phi(k) = \left|\F[\phi](k)\right|^{-2}
  \\\hline
  Dirac
    & \delta(z)
    & 1
    & 1\\
  Step
    & \Theta(z)
    & ik
    & k^2\\
  ReLU
    & (z)_+
    & -k^2
    & k^4\\
  Power ReLU
    & \frac{(z)_+^{\lambda-1}}{\Gamma(\lambda)}
    & (ik)^\lambda
    & |k|^{2\lambda}
    & \rule[-12pt]{0pt}{30pt}\\\hline
  Logistic Bump
    & \frac{\sigma e^{-\sigma z}}{\left(1 + e^{-\sigma z}\right)^2}
    & \frac{\sigma}{\pi k} \sinh(\frac{\pi k}{\sigma})
    & \frac{\sigma^2}{\pi^2 k^2}\sinh^2(\frac{\pi k}{\sigma})
    & \rule[-12pt]{0pt}{30pt}\\
  Sigmoid (Logistic)
    & \frac{e^{\sigma z}}{1+e^{\sigma z}}
    & \frac{i\sigma}{\pi} \sinh(\frac{\pi k}{\sigma})
    & \frac{\sigma^2}{\pi^2}\sinh^2(\frac{\pi k}{\sigma})
    & \rule[-12pt]{0pt}{30pt}\\
  SoftPlus
    & \frac{1}{\sigma} \ln(1+e^{\sigma z})
    & -\frac{\sigma k}{\pi}\sinh(\frac{\pi k}{\sigma})
    & \frac{\sigma^2 k^2}{\pi^2}\sinh^2(\frac{\pi}{\sigma}k)
    & \rule[-12pt]{0pt}{30pt}\\
  ``Power SoftPlus''
    & -\frac{1}{\sigma^n}\operatorname{Li}_n(-e^{\sigma z})
    & \frac{\sigma i^{n+1} k^n}{\pi} \sinh(\frac{\pi k}{\sigma})
    & \frac{\sigma^2 k^{2n}}{\pi^2}\sinh^2(\frac{\pi }{\sigma}k)
    & \rule[-12pt]{0pt}{30pt}\\\hline
  Cauchy
    & \frac{1}{\pi\sigma} \frac{1}{1+\left(\frac{z}{\sigma}\right)^2}
    & e^{\sigma|k|}
    & e^{2\sigma|k|}
    & \rule[-10pt]{0pt}{26pt}\\
  Arctangent
    & \frac{1}{\pi} \atan(\frac{z}{\sigma})
    & -ik e^{\sigma|k|}
    & k^2 e^{2\sigma|k|}
    & \rule[-10pt]{0pt}{26pt}\\
  Gaussian
    & \frac{1}{\sigma\sqrt{2\pi}} e^{-\frac{1}{2}\left(\frac{z}{\sigma}\right)^2}
    & e^{\frac{\sigma^2k^2}{2}}
    & e^{\sigma^2k^2}
    & \rule[-8pt]{0pt}{24pt}\\
  Erf
    & \frac{1}{2}\erf\left(\frac{z}{\sigma\sqrt{2}}\right)
    & -ik e^{\frac{\sigma^2 k^2}{2}}
    & k^2 e^{\sigma^2 k^2}
    & \rule[-8pt]{0pt}{24pt}\\
  G-function
    & \phi_n(z)\text{ as in \Cref{eq:gactivation}}
    & \exp\left[\frac{\sigma^n |k|^n}{n^{n-1}}\right]
    & \exp\left[\frac{2\sigma^n |k|^n}{n^{n-1}}\right]
    & \rule[-12pt]{0pt}{28pt}\\\hline
  SatReLU
    & (z)_+ - (z - \Delta)_+
    & \frac{-k^2}{1-e^{-i\Delta k}}
    & \frac{1}{2} \frac{k^4}{1 - \cos(\Delta k)}
    & \rule[-12pt]{0pt}{30pt}\\
  Wavepacket
    & \frac{\cos(\omega z)e^{-\frac{1}{2}\left(\frac{z}{\sigma}\right)^2}}{\sigma\sqrt{2\pi}}
    & \frac{2}{e^{-\frac{\sigma^2(k+\omega)^2}{2}} + e^{-\frac{\sigma^2(k-\omega)^2}{2}}}
    & \frac{4}{\left(e^{-\frac{\sigma^2(k+\omega)^2}{2}} + e^{-\frac{\sigma^2(k-\omega)^2}{2}}\right)^2}
    & \rule[-20pt]{0pt}{34pt}\\
  Rectangle
    & \operatorname{rect}(az)
    & \frac{k}{2\sin(\frac{k}{2a})}
    & \frac{k^2}{4\sin[2](\frac{k}{2a})}
    & \rule[-12pt]{0pt}{30pt}\\
  Triangle
    & \operatorname{tri}(az)
    & \frac{k^2}{4a\sin[2](\frac{k}{2a})}
    & \frac{k^4}{16a^2\sin[4](\frac{k}{2a})}
    & \rule[-12pt]{0pt}{30pt}\\
  Sinc
    & \frac{\sin(\pi a z)}{\pi a z}
    & \frac{a}{\operatorname{rect}\left(\frac{k}{2\pi a}\right)}
    & \frac{a^2}{\operatorname{rect}\left(\frac{k}{2\pi a}\right)}
    & \rule[-12pt]{0pt}{30pt}\\
  Squared Sinc
    & \frac{\sin[2](\pi a z)}{\pi^2 a^2 z^2}
    & \frac{a}{\operatorname{tri}\left(\frac{k}{2\pi a}\right)}
    & \frac{a^2}{\operatorname{tri}^2\left(\frac{k}{2\pi a}\right)}
    & \rule[-12pt]{0pt}{30pt}\\
  Half-Exponential
    & e^{-az}\theta(z)
    & a + ik
    & a^2 + k^2
    & \rule[-4pt]{0pt}{20pt}\\
  Hyperbolic Secant
    & \sech(a z)
    & \frac{a}{\pi} \cosh(\frac{\pi k}{2a})
    & \frac{a^2}{\pi^2} \cosh[2](\frac{\pi k}{2a})
    & \rule[-12pt]{0pt}{30pt}\\
  Log-Absolute
    & \log |z|
    & -\frac{|k|}{\pi}
    & \frac{k^2}{\pi^2}
    & \rule[-12pt]{0pt}{30pt}\\
  Oberhettinger I.84~\cite{Oberhettinger1973Fourier}
    & |z|^{-\frac{3}{2}}
      e^{-a/|z|}
    & \sqrt{\frac{a}{\pi}}
      e^{\sqrt{2ak}}
      \sec(\sqrt{2ak})
    & \frac{a}{\pi}
      e^{2\sqrt{2ak}}
      \sec[2](\sqrt{2ak})
    & \rule[-12pt]{0pt}{30pt}\\
  Oberhettinger I.70~\cite{Oberhettinger1973Fourier}
    & e^{-a|z|}(1-e^{-b|z|})^{\nu-1}
    & \frac
        {2b}
        {B\!\left(\nu,\frac{a-ik}{b}\!\right)
        +B\!\left(\nu,\frac{a+ik}{b}\!\right)
        }
    & \frac
        {4b^2}
        {\left[
           B\!\left(\nu,\frac{a-ik}{b}\!\right)
          +B\!\left(\nu,\frac{a+ik}{b}\!\right)
         \right]^2
        }
    & \rule[-12pt]{0pt}{30pt}\\
  Oberhettinger III.10~\cite{Oberhettinger1973Fourier}
    & (z-b)^{\nu-1}(z+b)^{-\nu-\frac{1}{2}}\llbracket z > b \rrbracket
    & \frac
        {\sqrt{b}}
        {2^{\nu-\frac{1}{2}}
         \Gamma(\nu)
         D_{-2\nu}\left(2\sqrt{ibk}\right)
        }
    & \frac
        {b}
        {2^{2\nu-1}
         \Gamma(\nu)^2
         D_{-2\nu}\left(2\sqrt{ibk}\right)^2
        }
    & \rule[-12pt]{0pt}{30pt}\\
\end{tabular}
}
\caption{Filters and Penalty Factors for various activations.
$\operatorname{Li}_n(\cdot)$ is the polylogarithm of order $n$ (for $n=1$, we
recover SoftPlus).
$\Gamma(\cdot)$ is the Gamma function, $B(\cdot,\cdot)$ is the Beta function,
and $D_\nu(\cdot)$ is the parabolic cylinder function.
}
\label{tab:activation}
\end{table*}

\newcommand{\actfig}[2]
{%
  \begin{tabular}{@{}c@{}}
    \includegraphics{#1-a.pdf}
    ~
    \includegraphics{#1-b.pdf}
    \\
    #2
  \end{tabular}
}
\begin{figure*}
  \resizebox{\textwidth}{!}{
  \begin{tabular}{@{}l@{}}
  \actfig
    {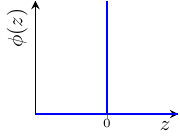}
    {Dirac}
  ~
  \actfig
    {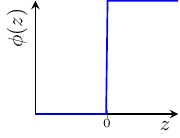}
    {Step}
  ~
  \actfig
    {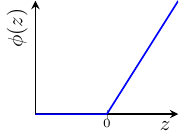}
    {ReLU}
  ~
  \actfig
    {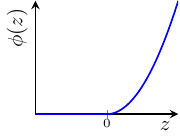}
    {Squared ReLU}
  \\[18mm]
  \actfig
    {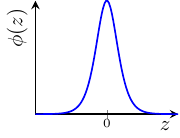}
    {Logistic Bump}
  ~
  \actfig
    {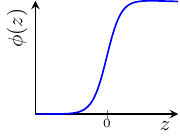}
    {Logistic}
  ~
  \actfig
    {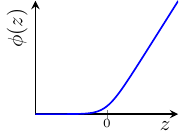}
    {SoftPlus}
  ~
  \actfig
    {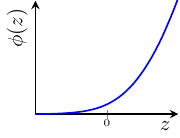}
    {$\operatorname{Li}_2$}
  \\[18mm]
  \actfig
    {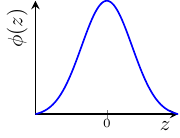}
    {Gaussian}
  ~
  \actfig
    {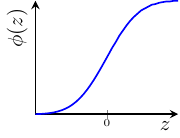}
    {Erf}
  ~
  \actfig
    {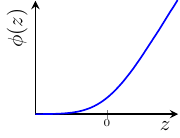}
    {Erf Ramp}
  ~
  \actfig
    {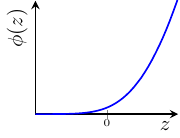}
    {Erf Quad Ramp}
  \\[18mm]
  \actfig
    {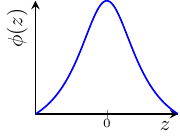}
    {Cauchy}
  ~
  \actfig
    {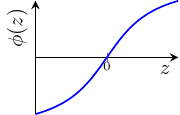}
    {Arctangent}
  ~
  \actfig
    {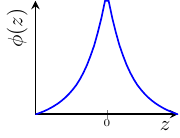}
    {G-function, $n=3$}
  ~
  \actfig
    {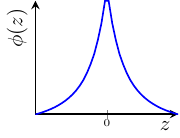}
    {G-function, $n=4$}
  \\[18mm]
  \actfig
    {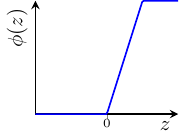}
    {SatReLU}
  ~
  \actfig
    {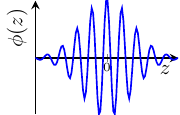}
    {Wavepacket}
  ~
  \actfig
    {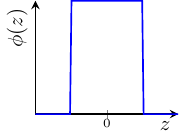}
    {Rectangle}
  ~
  \actfig
    {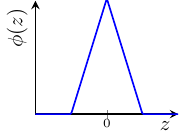}
    {Triangle}
  \\[18mm]
  \actfig
    {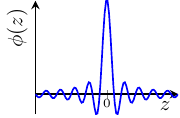}
    {Sinc}
  ~
  \actfig
    {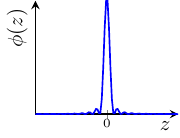}
    {Squared Sinc}
  ~
  \actfig
    {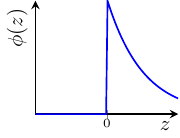}
    {Half-Exponential}
  ~
  \actfig
    {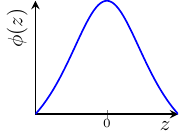}
    {Hyperbolic Secant}
  \\[18mm]
  \actfig
    {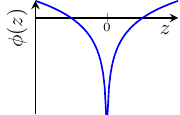}
    {Log-Absolute}
  ~
  \actfig
    {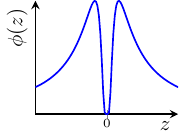}
    {Oberhettinger I.84~\cite{Oberhettinger1973Fourier}}
  ~
  \actfig
    {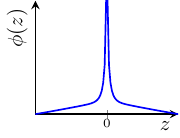}
    {Oberhettinger I.70~\cite{Oberhettinger1973Fourier}}
  ~
  \actfig
    {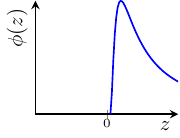}
    {Oberhettinger III.10~\cite{Oberhettinger1973Fourier}}
  \end{tabular}
  }
\caption{Activation functions $\phi(z)$ and their spectral penalty factors
$\rho_\phi(k)$.
For Sinc, and Squared Sinc, $\rho_\phi(k)$ is infinite outside the interval
$(-a,a)$, as indicated by the shaded region.}\label{fig:activation}
\end{figure*}

\subsection[Designing the Activation Function phi]{Designing the Activation Function $\phi$}\label{sec:choose}
Using these equations--especially \Cref{eq:sepsquares}--we can reverse engineer
an activation function from a desired spectral penalty factor $\rho(k)$:
\begin{equation}\label{eq:phirho}
  \phi_\rho(z) \triangleq \F^{-1}\left[\frac{1}{\sqrt{\rho(k)}}\right]\!(z)
\end{equation}

Note that we are inverting the squared magnitude in \Cref{eq:sepsquares}, a
many-to-one function; the inverse, which we have just written with
$\sqrt{\cdot}$, is therefore not unique.
For example, for a quadratic spectral penalty factor $\rho(k) \propto k^2$, we
can invert to $\F[\phi](k)^{-1} = ik$ to yield the Step activation, or we could
invert to $-|k|$ to yield the Log-Absolute activation.
In general, we can invert to $\zeta(k) k$ for any $\zeta : \RR\to\SS^1\subset
\CC$, which maps $k$ to any complex phase.

In the special case of the Power ReLU family's polynomial penalty factor,
writing $\rho(k) = k^{2\lambda}$, we have a closed form for the operator
$\L_{\phi,\vxi}^{-1}$: $\DD^{\lambda}_{+,\vxi}$, the (1-dimensional)
right-sided Riemann-Liouiville fractional derivative of order $\lambda$ applied
in the direction of $\vxi$; for integer values of $\lambda=n$, these are just
the directional derivatives $\partial_{\vxi}^n$.
(See~\Cref{sec:actderiv}.)
Using this, if we have a known order of derivative we wish to penalize, we can
choose the corresponding Power ReLU activation.
We can also use this to reason about activations built from (Power) ReLU
functions, such as the saturating ReLU $\operatorname{SatReLU}(z) = (z)_+ - (z
- \Delta)_+$, whose operator is $\L_{\phi,\vxi}^{-1} : f \mapsto
\sum_{j=0}^{\infty} \nabla^2 f(\cdot + \vxi\Delta j)$.

We can also use~\Cref{eq:phirho} to derive novel activation functions.
For example, we see from \Cref{tab:activation} that the Cauchy and Gaussian
activations have spectral penalty factors $\propto e^{2n|\frac{\sigma
k}{n}|^n}$ for $n=1$ and $n=2$, respectively.
Extending this pattern to higher $n$ yields a representation in terms of Meijer
G-functions:
\begin{equation}\label{eq:gactivation}
  \begin{IEEEeqnarrayboxm}[][c]{r;l}
    \phi_n(z)
      &\triangleq
         \frac{c_n}{\sigma}
         \left[
           G^{n,1}_{1,n}
             \left(
               \begin{matrix}\frac{n-1}{n}\\0,\frac{1}{n},\ldots,\frac{n-1}{n}\end{matrix}
             \middle|
               \frac{i^nz^n}{n\sigma^n}
             \right)
         \right.\\&\hspace{10mm}\left.
         + G^{n,1}_{1,n}
             \left(
               \begin{matrix}\frac{n-1}{n}\\0,\frac{1}{n},\ldots,\frac{n-1}{n}\end{matrix}
             \middle|
               \frac{(-i)^nz^n}{n\sigma^n}
             \right)
         \right]
  \end{IEEEeqnarrayboxm}
\end{equation}
where
\[ c_n
     \triangleq
       \frac
         {\sqrt{n}}
         {n^n 2^{\left\lceil\frac{n}{2}\right\rceil} \pi^{\frac{n+1}{2}}}.
\]
This activation with $n=3$ and $n=4$ is included in \Cref{fig:activation}.
Integrating any of these with respect to $z$ yields the Arctan and Erf
sigmoidal activation functions for $n=1$ and $n=2$; higher values of $n$ also
yield sigmoidal functions represented in terms of integrals of G-functions,
which tend toward a constant function as $n\to\infty$.
For $n>1$, additional integrals yield smooth approximations of the Power ReLU
family (for $n=1$, the fat tails of the Cauchy mean that the antiderivative of
$\atan(z)+\frac{\pi}{2}$ tends to $-\infty$ as $z\to-\infty$, rather than
approaching 0).

\subsection{The Radon Seminorm, Generalization, and the Curse of
Dimensionality.}\label{sec:curse}

We may also use the Radon seminorm and its Fourier interpretation to reason
about the generalization properties of learned functions.
Following \cite{Ongie2020Function}, we consider the contractions
$f_\varepsilon(\vx) \triangleq f(\vx/\varepsilon)$ for small $\varepsilon > 0$.
To connect these contractions to generalization, suppose $f(\cdot)$ is a
bump function centered at the origin.
Then, $f_\varepsilon(\cdot)$ is a sharper (or ``spikier'') bump at the origin,
as shown in \Cref{fig:contract}.
If our regularizer penalizes contractions, it then prefers less-spiky (i.e.\
smoother) bumps.
The calculations below are invariant to translations, and using the triangle
inequality gives the same threshold for spikiness penalization for a sum of
bumps.
Then, consider some function $g(\cdot)$ which can be represented as a sum of
bump functions centered at each datapoint.
If our regularizer fails to penalize contractions, any $g'(\cdot)$ with sharper
bumps will have lower cost.
Then, the lowest cost function will have infinitely-sharp bumps at datapoints
(i.e.\ a ``bed of nails'' fit), which will predict 0 for all inputs except the
training data and thus have no generalization at all.

\begin{figure}
  \resizebox{\columnwidth}{!}{\includegraphics{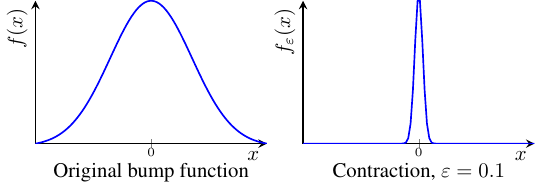}}
  \caption{A bump function $f(x)$ and its contraction, $f_\varepsilon(x)$ for $\varepsilon=0.1$.}
  \label{fig:contract}
\end{figure}

Therefore, we wish to show that our regularizer penalizes contractions.
Towards this goal, considering the Radon seminorm of $f_\varepsilon(\vx)$, we
have
\begin{IEEEeqnarray*}{r;l}
  \mathrlap{\|f_\varepsilon(\vx)\|_{\R,\phi,\eta_0}^2}
  \hspace{3mm}\\*
    &= \varepsilon^{-1}
       \!\!\!
       \int\limits_{\mathclap{\SS^{D-1}\times\RR}}
          \!\!
          \frac{\kappa_D^2}{\eta_0(\vxi,\varepsilon\gamma)}
          \left(
            \!
            \F_\gamma^{-1}
            \!
            \left[
              \frac{|\vartheta|^{D-1}}{\F_\gamma[\phi_\varepsilon](\vartheta)}
              \F_D[f](\vartheta\vxi)
            \right](\gamma)
          \right)^2
          \!\!\!
        \dd{\vxi}\dd{\gamma}\\
\shortintertext{%
where $\phi_\varepsilon(\cdot)=\phi(\varepsilon\cdot)$ is the \emph{dilation}
of $\phi$; $\eta_0(\cdot,\cdot)$ is likewise dilated in its second
argument.
If we suppose $\eta_0(\cdot,\cdot)$ is constant, this simplifies to}
    &= \varepsilon^{-1}
       \int\limits_{\mathclap{\RR^D}}
         k^{D-1}
         \rho_{\phi_\varepsilon}\mkern-2mu(k)
         \left|
           \F_D[f](\vk)
         \right|^2
       \dd{\vk}
\end{IEEEeqnarray*}
where $\rho_{\phi_\varepsilon}\mkern-2mu(k)$ is the spectral penalty factor
corresponding to the dilated activation $\phi_\varepsilon(\cdot)$.
To achieve generalization, we need $\lim_{\varepsilon\to0}
\|f_\varepsilon(\vx)\|_{\R,\phi,\eta_0}^2 = \infty$ so that ``spikier''
functions (small $\varepsilon$) are penalized more.
To have this, we need $\rho_{\phi_\varepsilon}\mkern-2mu(k) = o(\varepsilon)$,
which requires $\phi(\varepsilon z) = \omega\left(\varepsilon^{-1/2}\right)
\phi(z)$, which is independent of $D$.

If we remove the effects of the NN architecture (i.e.\ the $k^{D-1}$
in~\Cref{eq:sepsquares}) and the activation function (the
$1/|\F_\gamma[\phi](k)|^2$ term), we are left with minimizing $\|\F[f]\|_2 =
\|f\|_2$.
In this case, $\|f_\varepsilon\|_2 = \varepsilon^D\|f\|_2$, which leads to a
\emph{curse of dimensionality}, as spikier, non-generalizing functions are
preferred \emph{exponentially} in $D$.
Thus, the use of the shallow NN architecture and its relationship to the Radon
transform is indispensable in avoiding this curse of dimensionality.

\section{Adaptive Regime}\label{sec:adapt}

Recent work has shown the adaptive regime to be more
powerful~\cite{Lee2020Finite, Flesch2022Orthogonal}: as one transitions from
the kernel regime to the adaptive regime there is typically an increase in
generalization performance~\cite{Mehta2021Extreme}, which arises from the power
to adapt the zero-plane density $\eta_t(\vxi,\gamma)$ to the training data.
For example, an infinite-width deep convolutional NTK model achieves within 5\%
of a finite-width adaptive regime model~\cite{arora2019exact}.
In general, adaptive NNs can approximate a more complex class of functions than
the corresponding kernel-regime RKHS models~\cite{Ghorbani2020Neural}.

Learning dynamics in the adaptive regime are more complex, and so we do not
expect an equation as simple as~\Cref{eq:muopt} to hold.
Nevertheless, we will see that the $\theta_{\tRS}$ ``spline'' parameterization
is also useful in the adaptive regime.

The Fourier view will also turn out to offer insights and so we start by
considering the Fourier transform of a finite-width network:
\[
  \F_D[f_{\theta_{\tRS}}](\vk)
    = \sum_{j=1}^H
        \mu_i
        e^{-i\gamma_i \langle\vk,\vxi_i\rangle}
        \F_1[\phi_{\omega_i}](\langle\vk,\vxi_i\rangle)
        \delta_{\vxi_i}\!(\vk)
\]
where we define the ``Dirac-line'' distribution $\delta_{\vxi_i}\!(\vk)$ by
$\langle \delta_{\vxi_i},\psi\rangle \triangleq \int_\RR \psi(u\vxi_i) \dd{u}$.
Note that this distribution is only supported on lines through the origin parallel to
the $\vxi_i$.
The magnitude of the (complex) ``height'' along each line is given by
$\left|\mu_i\F_1[\phi_{\omega_i}](k)\right|$; for typical activation functions,
this will be concentrated at the origin.
Suppose that the target function $f^*$ has a periodic component in some
direction $\vxi^*$ with frequency $\lambda^*$.
Then, $\F_D[f^*](\vk)$ will have a local maximum at $\lambda^*\vxi^*$.
Then, the only way for $\F_D[f_{\theta_{\tRS}}](\vk)$ to well-approximate this
is if there are multiple $\vxi_i\approx\vxi^*$ with differing $\gamma_i$ such
that the complex sinusoids $e^{-i\gamma_i \langle\vk,\vxi_i\rangle}$
constructively interfere at the offset $\lambda^*$ in a way which counteracts
the decay of $\F_1[\phi](\langle\vk,\vxi_i\rangle)$.
This corresponds to having parallel zero-planes with spacing
$\propto\frac{1}{\lambda^*}$.

In the kernel regime, the spacing and alignment of zero-planes is governed by
the random initialization $\eta_0(\ldots)$.
For typical initalization schemes, this initial distribution is diffuse, so
that large $H$ is necessary to ensure that such parallel zero-planes exist.
Additionally, there is a curse of dimensionality at play here: as $D$
increases, the required $H$ grows exponentially.
In the adaptive regime, both $\vxi_i$ and $\gamma_i$ can be learned, so such
interference patterns can hypothetically be (approximately) orchestrated.
Accordingly, we examine these dynamics next.

\subsection[Dynamics of the Radon Spline Parameters, theta RS]{Dynamics of the Radon Spline Parameters, $\theta_{\tRS}$}
First, we consider training a ReLU network directly in the $\theta_{\tRS}$
parameterization.
Let $\widetilde{\vx} = (x_1,\ldots,x_D,1)$, $\widetilde{\vxi} =
(\xi_1,\ldots,\xi_D,-\gamma)$, and let $\C^{D-1}\triangleq\SS^{D-1}\times\RR$
denote the hyper-cylinder of possible breakplane coordinates.
This way, we have a single parameter that completely determines each
breakplane.
Let $\widetilde{\ell}(\widetilde{\vxi}_i|\ldots)$ denote the loss
$\widetilde{\ell}(\theta_{\tRS})$ with all parameters except
$\widetilde{\vxi}_i$ fixed.
Then, $\widetilde{\ell}(\widetilde{\vxi}_i|\ldots)$ is\footnote{Let
$\widecheck{\vxi}_i$ denote the embedding of $\widetilde{\vxi}_i$ into
$\RR^{D+1}$, and let $\widecheck{\ell}(\widecheck{\vxi}_i|\ldots)$ be the
extension of $\widetilde{\ell}(\widetilde{\vxi}_i|\ldots)$ to $\RR^{D+1}$.
Then, $\widecheck{\ell}(\widecheck{\vxi}_i|\ldots)$ is a CPW-Quadratic
real-valued function with non-negative curvature on $\RR^{D+1}$, and therefore
its gradient is (discontinuous, in general) piecewise linear (PWL).
These properties imply corresponding properties of
$\widetilde{\ell}(\widetilde{\vxi}_i|\ldots)$, but because its domain is
$\C^{D-1}$, the corresponding properties cannot be linearity and quadraticity.
For brevity, in the remainder of this section, we will treat
$\widetilde{\ell}(\widetilde{\vxi}_i|\ldots)$ as CPWQ with a PWL gradient.}
CPW-Quadratic, with piece boundaries consisting of the hyper-ellipses $\E_n
\triangleq \left\{\widetilde{\vxi}_i\mid
\langle\widetilde{\vxi}_i,\widetilde{\vx}_n\rangle = 0 \right\}$ formed by
intersecting the datapoint-associated planes $\P_n\triangleq \left\{
\vz\in\RR^{D+1}\mid \langle\vz,\widetilde{\vx}_n\rangle = 0\right\}$ with the
cylinder $\C^{D-1}$.
Then, the hyper-ellipse $\E_n$ corresponds to all breakplanes that pass through
datapoint $\vx_n$.

\begin{figure*}
\centering
\resizebox{0.77\textwidth}{!}{\includegraphics{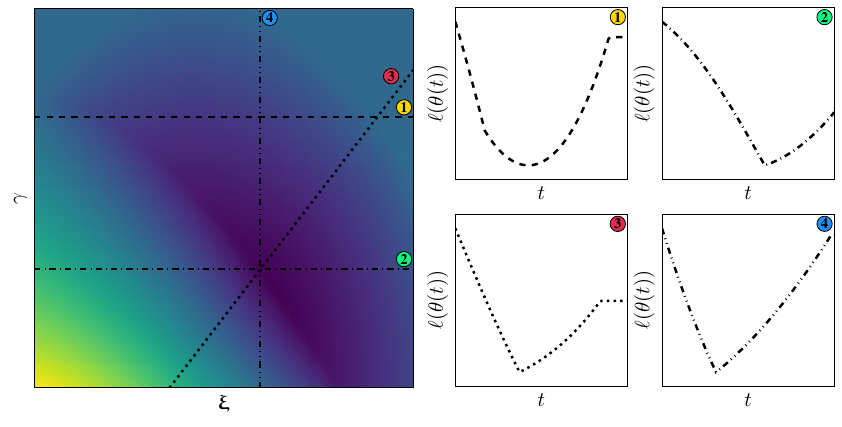}}
\resizebox{0.77\textwidth}{!}{\includegraphics{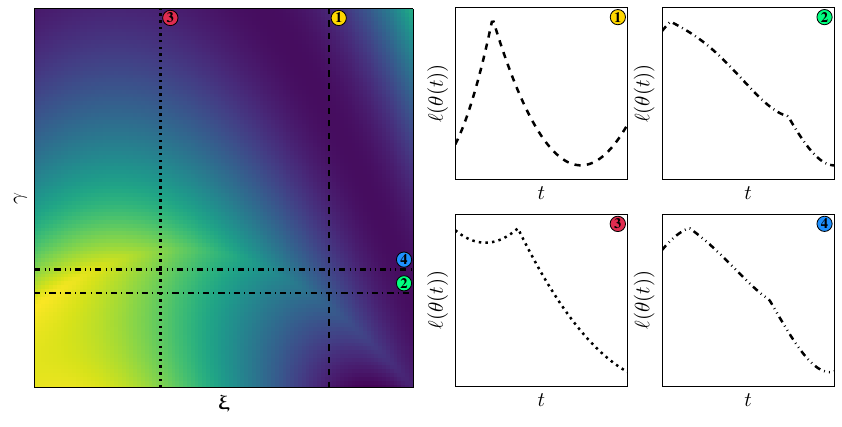}}
\resizebox{0.77\textwidth}{!}{\includegraphics{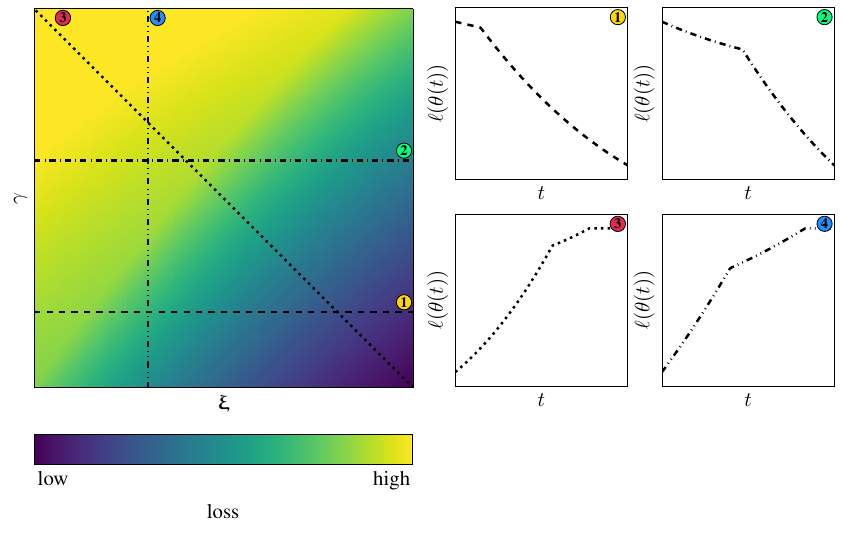}}
\caption{Local features of the loss surface slice
$\widetilde{\ell}(\widetilde{\vxi}_i|\ldots)$.
\textbf{Top}: a valley; \textbf{Middle}: a ridge; \textbf{Bottom}: a
pass-through crease.
\textbf{Left}: heatmap of the loss.
\textbf{Right}: 1-dimensional slices along numbered lines.}
\label{fig:loss2d}
\end{figure*}

\tikzset{
  , stylaba/.style={black,fill=Gold,draw,solid,thin,circle,inner sep=0.5pt,outer sep=3pt,font=\scriptsize\bfseries}
  , stylabb/.style={black,fill=SpringGreen,draw,solid,thin,circle,inner sep=0.5pt,outer sep=3pt,font=\scriptsize\bfseries}
  , stylabd/.style={black,fill=DodgerBlue,draw,solid,thin,circle,inner sep=0.5pt,outer sep=3pt,font=\scriptsize\bfseries}
}
\begin{figure*}
  \centering
  \resizebox{0.77\textwidth}{!}{\includegraphics{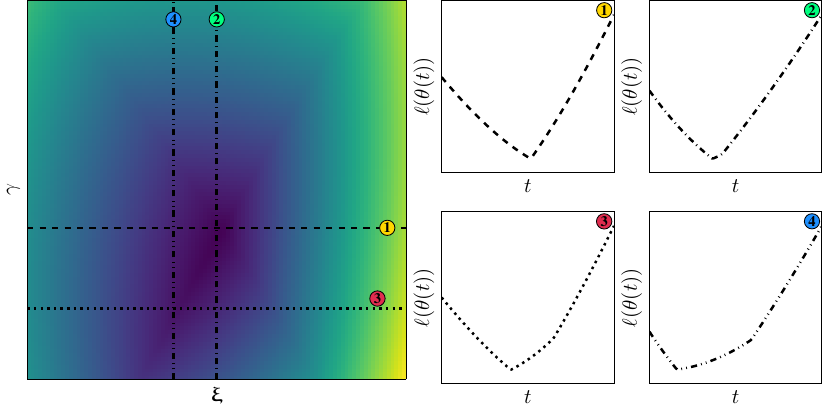}}
  \resizebox{0.77\textwidth}{!}{\includegraphics{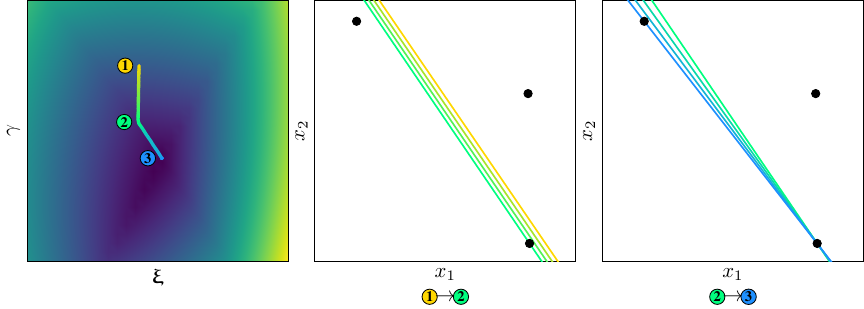}}
  \caption{
  Datapoint pinning: the region near the
  intersection of two datapoint ellipses $\E_n$ and $\E_m$ where both
  boundaries are valley floors.
  \textbf{Top Left:} heatmap of the loss.
  \textbf{Top Right:} 1-dimensional slices along numbered lines.
  \textbf{Bottom Left:} The parameter-space trajectory of a breakplane
  following gradient descent on $\widetilde{\ell}(\widetilde{\vxi}_i|\ldots)$.
  Starting at point \protect\tikz \protect\node [stylaba] {1};, the breakplane
  follows a nearly-vertical trajectory (i.e. almost all change is in
  $\gamma_i$) until it meets a valley floor at \protect\tikz \protect\node
  [stylabb] {2};, after which it remains confined to that valley floor, and is
  pinned by the corresponding datapoint.
  It then continues along the valley floor until it reaches the intersection
  point \protect\tikz \protect\node [stylabd] {3};, which is a local minima.
  \textbf{Bottom Right:} the trajectory of the breakplane in data space,
  showing that the breakplane first moves towards the bottom datapoint, then is
  constrained to rotate around that datapoint until it becomes pinned by the
  top datapoint as well.}
  \label{fig:pinning}
\end{figure*}

\subsubsection[Specializing to D=2]{Specializing to $D=2$}
In the case $D=2$, we have $\C^1=\SS^1\times\RR$, which is the ordinary
(infinite length) cylinder, and the piece boundaries $\E_n$ are ordinary
ellipses (embedded in $\RR^3$).

Consider a small neighborhood of a point on some boundary $\E_n$.
Then, $\widetilde{\ell}(\widetilde{\vxi}_i|\ldots)$ is smooth with non-negative
curvature (i.e.\ bowl-shaped) on either side of $\E_n$, with different
curvature on each side.
By continuity these two bowls must agree on the boundary $\E_n$.
This leads to the question: what shapes are possible \emph{along} that
boundary?
The answer is that $\widetilde{\ell}(\widetilde{\vxi}_i|\ldots)$ can take the
form of
\begin{itemize}
\item a ``ridge top'' (when the minima of the two bowls are contained on the
    same side of the boundary as their respective pieces);
\item a ``valley floor'' (when the minima are each contained on the opposite
    side); or
\item a ``pass-through crease'' (when both minima are contained on the same
    side).
\end{itemize}
Examples of these shapes are illustrated in \Cref{fig:loss2d}.
In the special case that one side of the boundary is the piece corresponding to
no active data (i.e. the active side of the breakplane points away from all
datapoints), this side will have zero curvature, and can take the form of
either a ``plateau'' such that loss is higher on the no-data side, or a flat
``basin'' such that loss is lower on the no-data side.
This last case is somewhat pathological, as breakplanes will be attracted to
the no-data configuration, and upon arrival will cease to receive gradient
updates.
Note that different regions of a single boundary $\E_n$ may have different
classifications.

Consider a datapoint $\vx_n$ and a value of $\widetilde{\vxi}_i$ near a region
where $\E_n$ is a valley floor.
Then, $\widetilde{\vxi}_i$ will be attracted to $\E_n$, and after a small
amount of training, will be confined to the valley floor, but may still have
gradient \emph{along} $\E_n$.
Such motion in parameter space corresponds to rotating the corresponding
breakline around the datapoint $\vx_n$; we say that neuron $i$ is \emph{pinned}
to $\vx_n$.

Then, consider additional datapoints: $\E_n$ will intersect any $\E_m$ at
exactly two points (unless $\vx_n=\vx_m$ in which case $\E_n=\E_m$).
If following the gradient along $\E_n$ does not reach a local minimum first, it
will eventually lead to one of the two intersection points with some $\E_m$.
If the region of $\E_m$ around this intersection point is also a valley floor,
then the intersection point will be a local minima where the breakline goes
through both datapoints.
This is illustrated in \Cref{fig:pinning}.

\subsubsection[Generalizing to D>2]{Generalizing to $D>2$}
The above analysis generalizes to higher dimensions as follows.
First, consider (for $D=2$) the neighborhood classification of some pinned
breakplane, $\widetilde{\vxi}_i\in\E_n$, and consider some contiguous region
$\N\subset\E_n$ containing $\widetilde{\vxi}_i$ for which the classification is
constant.
Then, $\N$ is a segment of an ellipse, and we can view it as an 1-manifold
embedded in $\C^1\subset\RR^3$.
Then, the neighborhood classification depends on the behavior of the loss as we
move along $\C^1$ normal to $\N$: e.g., if loss decreases then increases, we
have a valley floor.
Moving to $D>2$, $\N$ becomes a general $(D-1)$-manifold, but we can still move
normal to it, and we keep the same classification names as the $D=2$ case.

Assuming no datapoints are equal, the hyper-ellipses $\E_n$ intersect each
other in $(D-2)$-dimensional hyper-ellipses, which intersect as
$(D-3)$-dimensional hyper-ellipses, and so on until we have $D-1$
hyper-ellipses intersecting at 2 points.
Thus, the datapoint pinning phenomenon extends to higher dimensions: in a
region of $\E_n$ that is a valley floor, $\widetilde{\vxi}_i$ will be pinned to
$\vx_n$, but free to rotate around it ($D-1$ degrees of freedom).
At the intersection of $\E_n$ with another valley floor datapoint ellipse
$\E_m$, $\widetilde{\vxi}_i$ will be pinned to both $\vx_n$ and $\vx_m$ as
before, but will now have $D-2 > 0$ degrees of freedom.
For example, for $D=3$, the breakplane has 1 degree of freedom to rotate around
the line $\overline{\vx_n\vx_m}$.
We may repeat this logic until the hyperplane has no more degrees of freedom.
It is also possible for the motion along the intersection of hyper-ellipses to
lead to regions where one or more hyper-ellipse stops being an attractor, thus
restoring degrees of freedom, or for a regular local minimum to be reached
``between'' intersections.

\subsection[Dynamics of the Neural Network Parameters, theta NN]{Dynamics of the Neural Network Parameters, $\theta_{\tNN}$}
We now consider the dynamics of $\theta_{\tRS}$ during normal $\theta_{\tNN}$
training.
In this case, the $\theta_{\tRS}$ updates have an additional Jacobian factor,
and no longer correspond to gradient descent on $\widetilde{\ell}(\theta_{\tRS})$.
However, $\theta_{\tRS}$ will still trace out a continuous curve through
parameter space, and the value of $\widetilde{\ell}(\theta_{\tRS})$ still
determines the value of $\ell(\theta_{\tNN})$.
In particular, $\ell(\theta_{\tNN})$ can be constructed from
$\widetilde{\ell}(\theta_{\tRS})$ via the inclusion $(\vw,b,v) \triangleq
(\vxi,-\gamma,\mu)$ followed by copying the value along the $\alpha$-symmetry
hyperboloids.
Thus, local minima and $d$-dimensional valleys of $\widetilde{\ell}(\theta_{\tRS})$
map to 1-dimensional valleys and $(d+1)$-dimensional valleys in
$\ell(\theta_{\tNN})$, respectively.
These extended valleys are ``flat'' (have 0 gradient) along the
$\alpha$-symmetry curve, so if a parameter would be confined to a valley
according to $\theta_{\tRS}$ dynamics, it will still be confined according to
$\theta_{\tNN}$ dynamics, i.e.\ $\theta_{\tNN}$ dynamics admit the same cluster
formation dynamics including datapoint pinning.

Next, we consider the effects of the Jacobian factor on breakplane dynamics
under $\theta_{\tNN}$ training:

\begin{figure*}
  \centering
  \resizebox{0.8\textwidth}{!}{\includegraphics{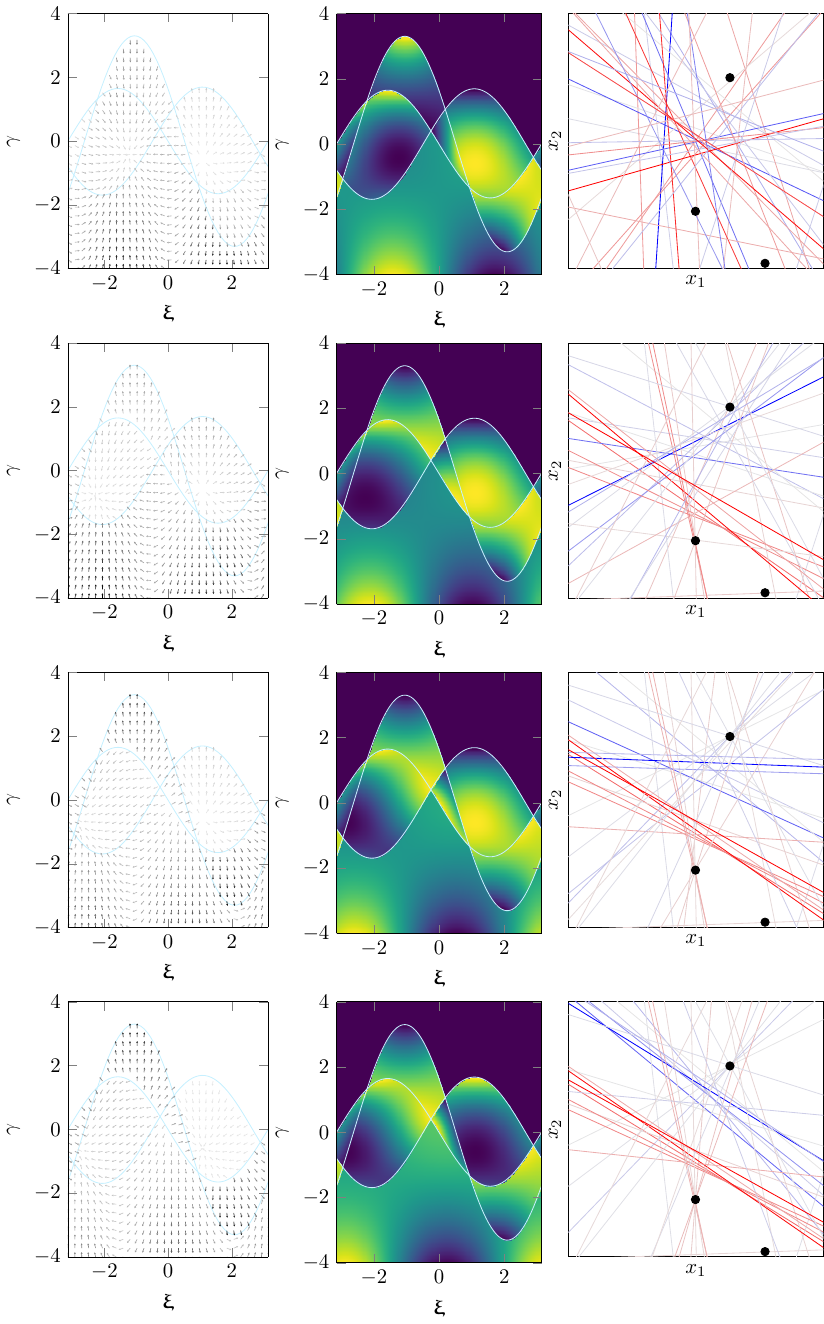}}
  \caption{Cluster formation: \textbf{Left:} $v_t(\vxi,\gamma)$ on a
  3-datapoint 2D example.
  \textbf{Center:} alternative visualization with maxima and minima
  corresponding to the sources and sinks of $v_t(\vxi,\gamma)$.
  \textbf{Right:} the breaklines, colored by delta-slope, and the 3 datapoints.
  Each row shows snapshots from different times throughout training.}
  \label{fig:stabpin}
\end{figure*}

\begin{center}
\resizebox{\columnwidth}{!}{
\begin{tabular}{ZO;M|O;MZ}
  \multicolumn{2}{c|}{$\theta_{\tRS}$ training}&
  \multicolumn{2}{c}{$\theta_{\tNN}$ training}\\
  \hline\rule{0pt}{1.5\normalbaselineskip}
  \dv{\vxi_i}{t}
    &=-\mu_i
       \left\langle
         \hat{\vepsilon}_t
       , \vX_i - \left\langle\vX_i,\vxi_i\right\rangle\vxi_i
       \right\rangle
  &\dv{\vxi_i}{t}
    &=-\frac{\mu_i}{\omega_i^2}
       \left\langle
         \hat{\vepsilon}_t
       , \vX_i - \left\langle\vX_i,\vxi_i\right\rangle\vxi_i
       \right\rangle\\
  \dv{\gamma_i}{t}
    &= \mu_i
       \left\langle
         \hat{\vepsilon}_t
       , \vec{1}_i
       \right\rangle
  &\dv{\gamma_i}{t}
    &= \frac{\mu_i}{\omega_i^2}
       \left\langle
         \hat{\vepsilon}_t
       , \vec{1}_i + \gamma_i\left\langle\vX_i,\vxi_i\right\rangle
       \right\rangle\\
  \\[-0.7\normalbaselineskip]
  \hline\rule{0pt}{1.6\normalbaselineskip}
  \dv{\widetilde{\vxi}_i}{t}
    &=-\mu_i
       \left\langle
         \!
         \hat{\vepsilon}_t
       , \widetilde{\vX}_i
        -\left\langle\vX_i,\vxi_i\right\rangle\widecheck{\vxi}_i
         \!
       \right\rangle
  &\dv{\widetilde{\vxi}_i}{t}
    &=-\frac{\mu_i}{\omega_i^2}
       \left\langle
         \!
         \hat{\vepsilon}_t
       , \widetilde{\vX}_i
        -\left\langle\vX_i,\vxi_i\right\rangle\widetilde{\vxi}_i
         \!
       \right\rangle\\[3mm]
     &\triangleq
       -\mu_i
        v_t^{\tRS}(\widetilde{\vxi}_i)
    &&\triangleq
       -\frac{\mu_i}{\omega_i^2}
        v_t(\widetilde{\vxi}_i)
\end{tabular}
}
\end{center}
where $\widecheck{\vxi}_i\triangleq(\xi_1,\ldots,\xi_D,0)$.
Thus, the effect of the Jacobian factor is to introduce the scalar factor
$1/\omega_i^2$ to all dynamics, and the
$\gamma_i\left\langle\vX_i,\vxi_i\right\rangle$ term in the $\gamma_i$
dynamics.
Then, the last line shows that each breakplane moves with its own (scalar) rate
multiplier (for $\theta_{\tNN}$ training, $-\mu_i/\omega_i^2$) according to the
shared vector field $v_t(\cdot)$.
Borrowing terminology from the study of fluid dynamics, $v_t$ defines a
velocity flow vector field.
Note that $v_t^{\tRS}(\cdot)$ and $v_t(\cdot)$ are piecewise quadratic, with
the same piece structure based on the datapoint hyper-ellipses $\E_n$ as
$\widetilde{\ell}(\widetilde{\vxi}|\ldots)$.
Let $\sigma(\widetilde{\vxi})$ denote the region of $\C^{D-1}$ containing
$\widetilde{\vxi}$ (corresponding to a given activation pattern on the training
data), and let $\vX_\sigma$ and $\widetilde{\vX}_\sigma$ denote the masked data
and masked augmented data for a given region $\sigma$.
Then, restricting to a region $\sigma$, we can write the shared vector field
for $\theta_{\tNN}$ training as
\begin{IEEEeqnarray*}{r;l}
  v_t(\sigma,\widetilde{\vxi})
    &= \left\langle
         \hat{\vepsilon}_t
       , \widetilde{\vX}_{\sigma}
        -\left\langle\vX_{\sigma},\vxi\right\rangle\widetilde{\vxi}
       \right\rangle\\
    &\triangleq
       \widetilde{\vd}_{\sigma}
      -\langle\vd_{\sigma},\vxi\rangle\widetilde{\vxi}
\end{IEEEeqnarray*}
Inspection reveals that $\mathrlap{\widetilde{\phantom{\vxi}}}\vxi^*_\sigma
\triangleq
\frac{(\vd_\sigma:-\langle\hat{\epsilon}_t,\vec{1}\rangle)}{\|\vd_\sigma\|_2}$
is a sink for the vector field $v_t(\sigma,\cdot)$, and the antipodal point
$-\mathrlap{\widetilde{\phantom{\vxi}}}\vxi^*_\sigma$ is a source.
These points are thus attractor and repeller for $\widetilde{\vxi}$ dynamics,
when $-\frac{\mu_i}{\omega_i^2}$ is positive (and repeller and attractor,
respectively, when it is negative).

Observing that $\vd_\sigma\triangleq\langle\hat{\vepsilon}_t,\vX_\sigma\rangle$
is the direction along which active data is maximally correlated with error, we
see that $v_t(\sigma,\cdot)$ is driving $\vxi$ to maximize correlation with
error, so that changes in the delta-slope $\mu_i$ will maximally reduce error.
Similarly, in the $\theta_{\tRS}$ case, $\gamma$ is driven to reduce the net
error $\langle\hat{\vepsilon}_t,\vec{1}_\sigma\rangle$.
However, in the $\theta_{\tNN}$ case, $\gamma$ and $\mu$ are coupled by the
Jacobian, so that $\gamma$ will get no gradient at
$\gamma=-\frac{\langle\hat{\epsilon}_t,\vec{1}\rangle}{\langle\vd_\sigma,\vxi\rangle}$,
even if $\langle\hat{\vepsilon}_t,\vec{1}_\sigma\rangle$ is non-zero.
Conversely, if $\langle\vd_\sigma,\vxi\rangle$ were to remain small in
magnitude, $\gamma$ would be driven to larger and larger magnitudes.

In general, $\mathrlap{\widetilde{\phantom{\vxi}}}\vxi^*_\sigma$ and
$-\mathrlap{\widetilde{\phantom{\vxi}}}\vxi^*_\sigma$ need not be contained in
$\sigma$, in which case any flow along $v_t(\sigma,\cdot)$ will lead to some
piece boundary $\E_n$.
Across this boundary is some other sector $\sigma'$, with its own
$v_t(\sigma',\cdot)$.
We may then ``glue'' the $v_t(\sigma,\cdot)$ together into $v_t(\cdot)$.
Maennel \emph{et al.} formalized this (in terms of $\theta_{\tNN}$) using the
formalism of stratified vector fields in~\cite{Maennel2018Gradient}.
The case that flow converges towards the same boundary $\E_n$ on $\sigma'$
means that flow will subsequently be confined to $\E_n$.
This corresponds to the ``valley floor'' and ``pinning'' phenomena discussed
above.

\begin{figure*}
  \resizebox{\textwidth}{!}{\includegraphics{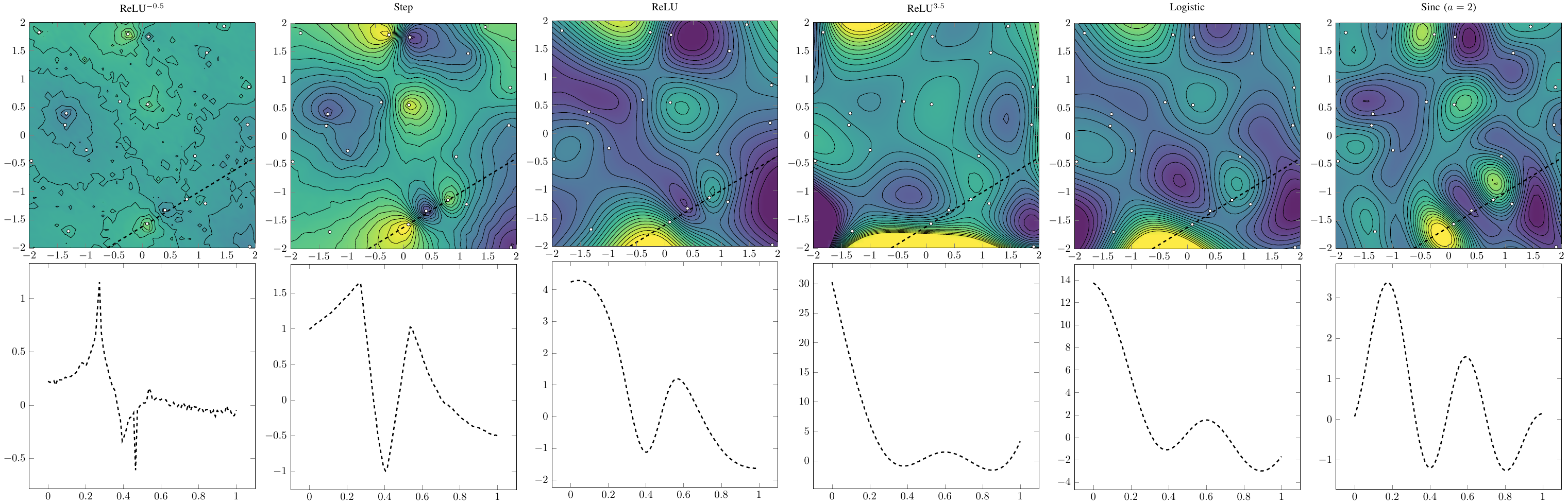}}
  \caption{Example fits on 20 uniform random points.
  Top: heatmap and contour plot, input data as white points; a dashed black
  line marks a 1-dimensional slice through two datapoints which also passes
  near a third.
  Bottom: the fit function along the 1-dimensional slice.}
  \label{fig:gd_vs_conv}
\end{figure*}

As training progresses and residuals change, the attractors
$\pm\mathrlap{\widetilde{\phantom{\vxi}}}\vxi^*_\sigma$ will move around.
If this motion is slow enough relative to the motion of breakplanes,
breakplanes within $\sigma$ will ``flow together'' towards
$\pm\mathrlap{\widetilde{\phantom{\vxi}}}\vxi^*_\sigma$.
\Cref{fig:stabpin} Shows $v_t(\cdot)$ and breakplane parameters throughout
training on a toy 2D dataset of 3 points.

In practice the attractors
$\pm\mathrlap{\widetilde{\phantom{\vxi}}}\vxi^*_\sigma$ move enough that
breakplanes rarely form ``hard clusters'' where a large set of breakplanes
perfectly align with each other.
Instead, breakplanes flow towards the moving target with decreasing speed as
error is reduced and thus the $\|\vd_\sigma\|_2$ decrease; eventually,
$\|\vd_\sigma\|_2\approx0$ and motion stops.
This leads to breakplanes forming ``soft'' or ``smeared-out'' clusters near the
final $\pm\mathrlap{\widetilde{\phantom{\vxi}}}\vxi^*_\sigma$.
An exception to this smearing is when datapoint pinning persists (i.e. when two
adjacent attractors $\pm\mathrlap{\widetilde{\phantom{\vxi}}}\vxi^*_\sigma$ and
$\pm\mathrlap{\widetilde{\phantom{\vxi}}}\vxi^*_{\sigma'}$ persist on opposite
sides of a boundary), leading to clusters that are sharply concentrated on the
boundary, but potentially smeared ``along'' the boundary.

\subsection{Other Activation Functions}\label{sec:adaptgen}
These phenomena generalize to non-ReLU networks, with similar loss landscapes
and shared vector fields $v_t(\widetilde{\vxi})$, leading to similar cluster
formation behaviour.
For example, the loss landscape and dynamics with the SoftPlus activation will
look essentially unchanged far from piece boundaries.
Near piece boundaries, the non-differentiable cusps and the discontinuities in
$v_t(\cdot)$ are relaxed to a smooth landscape smooth field that interpolates
between the two pieces.
In the ``valley floor'' case, the discontinuous reversal of direction is
replaced by a smooth field such that the component perpendicular to the
boundary is zero at the boundary, leading to the same cluster formation
phenomenon.
However, because the field is smooth, any particle following it will slow down
as it approaches the boundary.
As residuals are reduced throughout training, the field $v_t(\cdot)$ decreases
in magnitude, exacerbating the slowing down and leading to looser clusters that
are only ``softly'' pinned to datapoints.
Generalizing in a different direction, the $v_t(\cdot)$ for the SatReLU
activation will have two discontinuities, corresponding to the lower and upper
bounds.
Thus, each datapoint will be associated with two piece boundaries.
More unconventional activation functions with complex shapes (such as the more
atypical entries in~\Cref{tab:activation,fig:activation}) will lead to more
complex loss landscapes and fields $v_t(\cdot)$.

\section{Experiments}

\subsection{Kernel Regime}

\begin{table}
\centering
\resizebox{\columnwidth}{!}{
\rowcolors{2}{black!7}{white}
\begin{tabular}{r|l|l|l}
Activation                                                                                  & Relative Error      & GD Training Time        & Convex Opt.\\
                                                                                            &                     &                         & Training Time\\\hline
Step                                                                                        & $3.79\% \pm 0.16\%$ & $0.44 \pm 0.13$         & $40.56 \pm 0.89$\\
ReLU                                                                                        & $5.53\% \pm 0.15\%$ & $10.30 \pm 0.90$        & $2.01 \pm 0.07$\\
ReLU$^2$                                                                                    & $5.62\% \pm 0.16\%$ & $100.36 \pm 5.66$       & $2.26 \pm 0.18$\\
ReLU$^{3.5}$                                                                                & $6.49\% \pm 0.33\%$ & $4{,}319.53 \pm 139.69$ & $2.11 \pm 0.06$\\
Logistic                                                                                    & $3.60\% \pm 0.08\%$ & $85.25 \pm 3.49$        & $3.21 \pm 0.14$\\
Atan                                                                                        & $4.37\% \pm 0.17\%$ & $170.41 \pm 4.17$       & $3.18 \pm 0.12$\\
Erf                                                                                         & $2.01\% \pm 0.04\%$ & $172.31 \pm 6.62$       & $3.19 \pm 0.13$\\
Sinc ($a\!=\!2$)                                                                            & $4.48\% \pm 0.09\%$ & $0.22 \pm 0.07$         & $130.29 \pm 11.16$\\
Sinc ($a\!=\!0.75$)                                                                         & $0.97\% \pm 0.06\%$ & $114.68 \pm 5.10$       & $3.82 \pm 0.09$\\
Sinc ($a\!=\!10$)                                                                           & $7.07\% \pm 0.22\%$ & $0.21 \pm 0.05$         & $235.21 \pm 20.72$\\
\begin{tabular}{@{}r@{}} Wavepacket\\($\omega\!=\!10, \sigma\!=\!0.7$)\end{tabular}         & $5.29\% \pm 0.07\%$ & $0.15 \pm 0.06$         & $255.49 \pm 13.31$\\
\begin{tabular}{@{}r@{}} Wavepacket\\($\omega\!=\!20, \sigma\!=\!0.1$)\end{tabular}         & $3.82\% \pm 0.07\%$ & $0.38 \pm 0.03$         & $87.87 \pm 2.44$\\
\begin{tabular}{@{}r@{}} Wavepacket\\($\omega\!=\!50, \sigma\!=\!2$) \end{tabular}          & $4.03\% \pm 0.03\%$ & $0.19 \pm 0.04$         & $312.67 \pm 53.98$\\
Cauchy                                                                                      & $1.67\% \pm 0.02\%$ & $29.30 \pm 3.22$        & $3.21 \pm 0.05$\\
Half-Exp                                                                                    & $3.00\% \pm 0.13\%$ & $0.29 \pm 0.03$         & $101.79 \pm 9.84$\\
\begin{tabular}{@{}r@{}} Oberhettinger\\III.10~\cite{Oberhettinger1973Fourier}\end{tabular} & $2.84\% \pm 0.06\%$ & $0.28 \pm 0.04$         & $68.09 \pm 0.74$
\end{tabular}
}
\caption{Relative error between GD and convex optimization fits with training
time in seconds for each method; mean $\pm$ standard deviation from 5
replications.}
\label{tab:gd_vs_conv}
\end{table}

\begin{figure*}
  \resizebox{\textwidth}{!}{\includegraphics{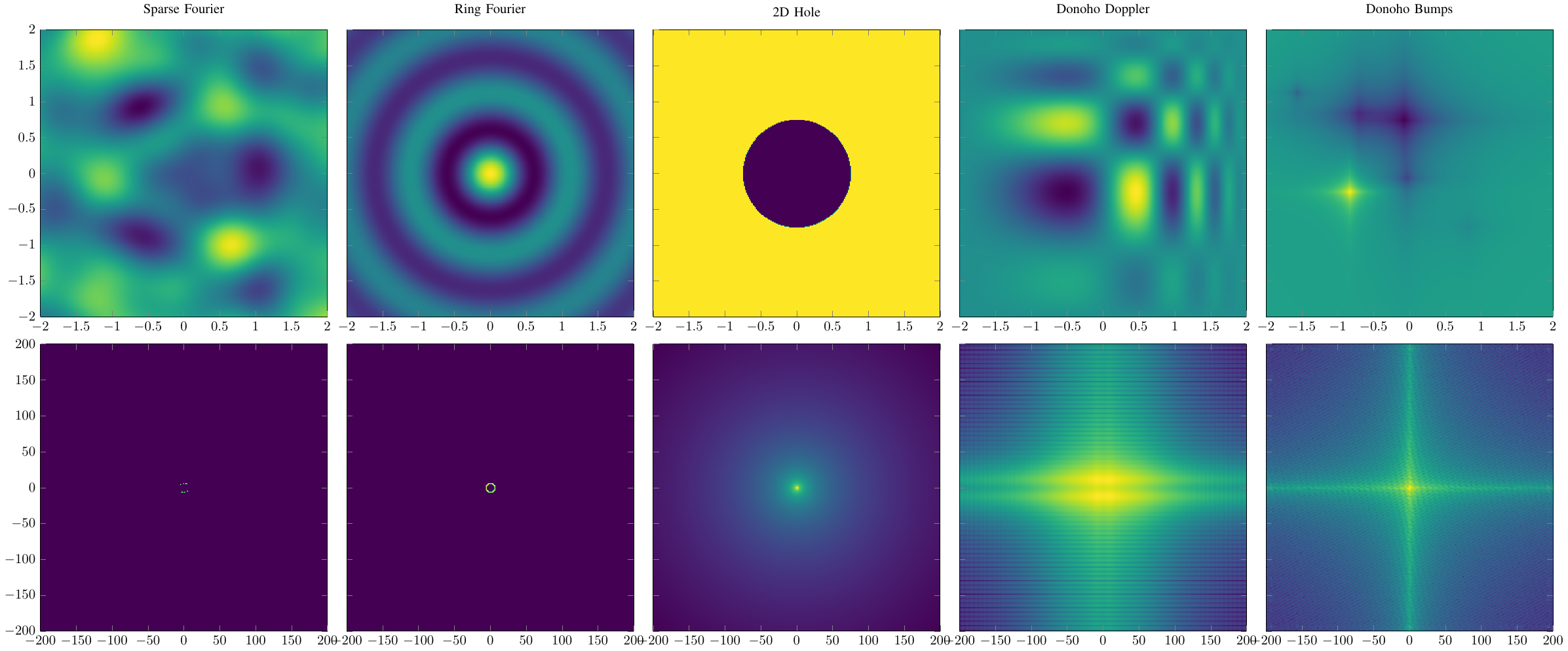}}
  \caption{Experiment~\ref{ex:kernelconv} Targets. Top row: the 5 targets. Bottom row: the log magnitude of the Fourier transform of each target.}
  \label{fig:e2targets}
\end{figure*}

\experiment First we demonstrate that \Cref{eq:muopt} is accurate
when training with finite step size and stopping before perfect zero error.
We use a simple dataset consisting of $N=20$ points uniformly drawn from a
$4\times 4\times 4$ cube.
We then train a series of wide ($H=20,000$) single layer fully connected neural
networks with MSE loss using full-batch Adam in the pure kernel regime (i.e.\
with input weights and biases frozen) to predict the third dimension given the
first two ($D=2$).
We initialize with $\|\vw_i\|_2 = 1$ (i.e.\ $\omega_i = 1$) deterministically
for all $i$, i.e.\ $\vw_i$ is uniform on the unit circle; we initialize $b_i$
uniform on $[-a,a]$ where $a > \max_n \|\vx_n\|$; $v_i$ is initialized
deterministically to 0 so that $\vmu_0=\vec{0}$.
We vary the activation function across the set $\{$Step, ReLU, ReLU${}^2$,
ReLU${}^{3.5}$, Logistic, Arctan, Erf, Sinc (with each of
$a=2$,$a=0.75$,$a=10$), Wavepacket (with each of
$\omega=10,\sigma=0.7$;$\omega=20,\sigma=0.1$;$\omega=50,\sigma=2$)$\}$
(see~\Cref{tab:activation} for definitions of these functions).
We train to an error of $10^{-6}$ (except the ReLU${}^{3.5}$, which we stopped
at 2.5 million iterations and an error of $1.3\times10^{-6}$).
We then use the convex optimization library CVXPY to solve the optimization
\begin{IEEEeqnarray*}{c}
  \hat{\vmu}
    = \argmin_{\vmu} \|\vmu-\vmu_0\|_2^2\\
  \text{ s.t.\ }\
  \frac{1}{N}
  \sum_{n=1}^N
    \left(y_n - \Phi\left(\vx_n;(\vxi_i,\gamma_i,\omega_i)_{i=1}^H\right) \vmu\right)^2
  \le \varepsilon
\end{IEEEeqnarray*}
for each activation function, where $\varepsilon$ is the error achieved by GD
for the same activation function.
This relaxes the hard equality constraint of~\Cref{eq:muopt} to account for the
early stopping of GD and to allow for finite numerical precision.

\Cref{tab:gd_vs_conv} shows the relative error (mean and standard deviation
from 5 replications), defined as the absolute difference between GD and the
convex optimization fits, averaged over the convex hull of the training data,
divided by the maximum value attained by the GD fit in the same area.
Relative error values of a few percentage points is typical.
with the largest values given by ReLU${}^{3.5}$ and Sinc ($a=10$).
ReLU${}^{3.5}$ has extreme growth that leads to difficulty training and high
sensitivity to small changes in parameters, so that even small numerical
innacuracies could lead to large relative error; Sinc ($a=10$) yields noisy,
highly oscillatory fits, such that high relative error could be achieved by a
slight misalignment.
Even these only have 6.49\% and 7.07\% relative error, respectively.
\Cref{tab:gd_vs_conv} also shows the training time for GD and convex
optimization.
For ReLU family, Sigmoids, Sinc ($a=0.75$), and Cauchy: the convex optimization
is faster, sometimes \emph{much} faster (e.g. ReLU${}^{3.5}$ is 2000$\times$
faster).
For Step, Sinc ($a=2$), Sinc ($a=10$), Wavepacket, Half-Exp, Oberhettinger: the
convex optimization is slower, up to 1600$\times$ slower.
Notably, the cases where convex is slower all have very fast GD training.
A few convex optimization fits are shown in~\Cref{fig:gd_vs_conv} to
demonstrate the qualitative effects of the various activation functions.

\begin{table*}
\newcommand{\bl}[1]{\textcolor{blue}{#1}}
\newcommand{\rr}[1]{\textcolor{red}{#1}}
\centering
\resizebox{\textwidth}{!}{
\begin{tabular}{l|c|c|c|c|c|c|c|c|c|c}
                                     & \multicolumn{2}{c|}{Sparse Fourier}
                                     & \multicolumn{2}{c|}{Ring Fourier}
                                     & \multicolumn{2}{c|}{2D Hole}
                                     & \multicolumn{2}{c|}{Donoho Doppler}
                                     & \multicolumn{2}{c}{Donoho Bumps}\\
Activation                                  & Train         & Test          & Train         & Test          & Train         & Test          & Train         & Test          & Train         & Test\\\hline
Step                                        & 8.31e-08      & 5.91e-01      & \bl{4.97e-08} & 9.08e-03      & 5.48e-08      & \bl{1.74e-02} & 5.00e-08      & 1.24e-02      & 5.14e-08      & 7.34e-04 \\
ReLU                                        & 5.07e-08      & 1.96e-01      & 5.00e-08      & 2.67e-03      & \bl{5.00e-08} & 1.96e-02      & \bl{4.99e-08} & 8.96e-03      & \bl{5.00e-08} & 5.09e-04 \\
ReLU${}^2$                                  & 5.03e-08      & 4.57e-02      & 4.99e-08      & 1.16e-03      & 5.00e-08      & 2.27e-02      & 5.00e-08      & \bl{7.37e-03} & 5.00e-08      & 2.31e-03 \\
ReLU${}^{3.5}$                              & 5.16e-08      & 2.53e-02      & 5.01e-08      & 3.32e-04      & 6.17e-08      & 3.77e-02      & 6.98e-08      & 1.23e-02      & 1.12e-06      & 2.71e-01 \\
Logistic                                    & 5.10e-08      & \bl{1.75e-02} & 5.00e-08      & \bl{2.54e-05} & 9.73e-06      & 5.89e+00      & 5.16e-08      & 7.70e-02      & 9.79e-05      & 2.83e+03 \\
Atan                                        & 5.01e-08      & 2.47e-02      & 5.00e-08      & 3.48e-04      & 1.11e-05      & 3.33e+00      & 5.59e-08      & 5.73e-02      & 6.66e-05      & 7.77e+02 \\
Erf                                         & ---           & ---           & ---           & ---           & 9.54e-02      & 8.75e+00      & \rr{2.03e-02} & 2.19e+00      & 1.31e-02      & 1.66e+02 \\
Cauchy                                      & \bl{5.01e-08} & 2.28e-02      & 5.00e-08      & 1.63e-04      & 6.97e-07      & 3.91e+00      & 4.99e-08      & 6.44e-02      & 4.00e-05      & 6.21e+03 \\
Sinc; $a\!=\!2$                             & \rr{2.65e-05} & 6.85e+02      & 5.00e-08      & 3.78e-05      & 1.71e-03      & \rr{3.64e+04} & 8.38e-06      & \rr{1.48e+01} & 2.64e-04      & \rr{4.01e+04} \\
Sinc; $a\!=\!0.75$                          & ---           & ---           & ---           & ---           & 7.49e-02      & 3.35e-01      & ---           & ---           & \rr{1.39e-02} & 2.12e+01 \\
Sinc; $a\!=\!10$                            & 1.68e-07      & 4.39e+00      & 4.98e-08      & 2.57e-02      & 5.09e-08      & 6.04e-02      & 5.02e-08      & 3.12e-02      & 5.13e-08      & 2.85e-03 \\
Wavepacket; $\omega\!=\!10, \sigma\!=\!0.7$ & 5.15e-08      & \rr{2.99e+03} & 5.00e-08      & \rr{1.52e-01} & 5.00e-08      & 1.79e+01      & 5.00e-08      & 1.09e+00      & 5.30e-08      & 1.55e+04 \\
Wavepacket; $\omega\!=\!20, \sigma\!=\!0.1$ & 2.17e-07      & 2.64e+01      & \rr{1.82e-07} & 4.35e-02      & 7.00e-08      & 4.39e-01      & 1.36e-07      & 6.91e-02      & 6.33e-08      & 2.78e-02 \\
Wavepacket; $\omega\!=\!50, \sigma\!=\!2$   & 7.50e-08      & 1.71e+02      & 5.00e-08      & 1.37e-01      & 5.02e-08      & 2.40e+00      & 5.01e-08      & 3.36e-01      & 5.00e-08      & 7.48e-01 \\
Half-Exp                                    & 8.63e-08      & 6.22e-01      & 5.75e-08      & 8.69e-03      & 5.35e-08      & 1.94e-02      & 5.00e-08      & 1.30e-02      & 5.15e-08      & 7.53e-04 \\
Oberhettinger III.10~\cite{Oberhettinger1973Fourier} & 8.23e-08 & 1.83e-01  & 5.33e-08      & 3.36e-03      & 5.00e-08      & 1.81e-02      & 5.00e-08      & 9.38e-03      & 5.00e-08      & \bl{5.00e-04}
\end{tabular}
}
\caption{Training and test error for Experiment~\ref{ex:kernelconv}: shallow
NNs trained with CVXPY to fit each of 5 targets.
The best value in each column is shown in blue, and the worst is shown in red.}
\label{tab:kernel}
\end{table*}

\begin{figure*}
  \resizebox{\textwidth}{!}{\includegraphics{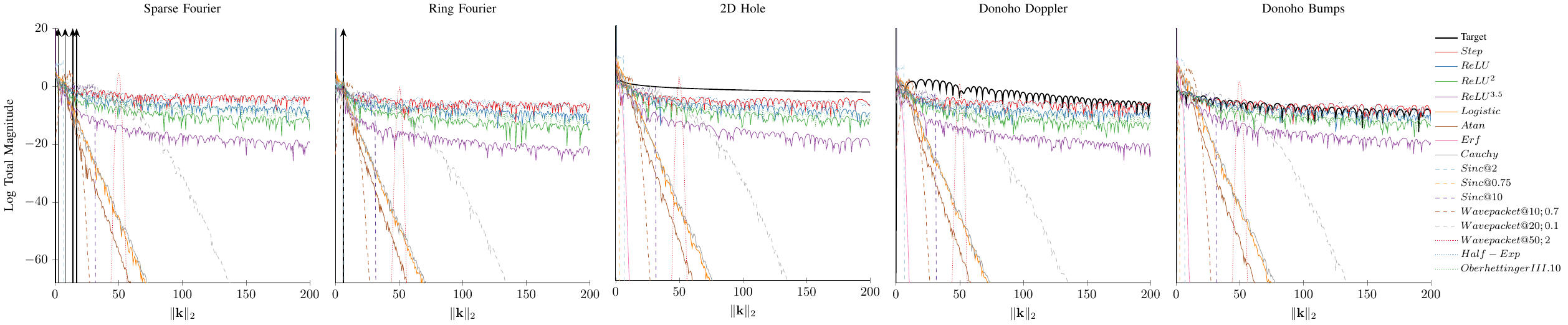}}
  \caption{Log total magnitude $\log M(r)$ for each target
  and activation function.
  Each target is shown in black, with Dirac distributions shown as vertical
  arrows.}
  \label{fig:total_mag}
\end{figure*}

\experiment \label{ex:kernelconv} We then use the (often much faster) convex
optimization, with a fixed mean square error of $5\times10^{-8}$, and the same
initialization and set of activation functions as Experiment 1 to fit several
targets chosen to exhibit a variety of Fourier features including two targets
based on Donoho's \emph{spatially nonhomogeneous}
functions~\cite{Donoho1995Adapting}:
\begin{enumerate}
\item The sum of six plane waves of random phase, frequency, direction, and
    amplitude.
    The Fourier transform of each plane wave is a pair of Dirac distributions
    aligned with the direction of the wave, with distance determined by the
    frequency.
    The target's Fourier transform is a weighted sum of six such distributions.
\item A radially symmetric Bessel function, $J_0(2\pi\|\vx\|_2)$.
    The Fourier transform of this target is a ``Dirac ring'' distribution whose
    support is a circle of radius $2\pi$.
\item A radially symmetric generalization of the Heaviside step function, which
    is 0 on a disc of radius $\frac{3}{4}$, and 1 outside of the disc.
    The Fourier transform of this function is proportional to
    $\frac{1}{\|\vk\|_2^2}$.
\item A 2D function inspired by Donoho's ``doppler''
    function~\cite{Donoho1995Adapting}, consisting of the product of two
    axis-aligned 1D functions which are sinusoids with smoothly-varying
    frequency multiplied by a smooth envelope.
\item A 2D generalization of Donoho's ``bumps''
    function~\cite{Donoho1995Adapting}, consisting of the weighted sum of
    several copies of the sharply pointed function $\frac{1}{|x_1|^4|x_2|^4}$
    with random translations, scales and (positive and negative) weights.
\end{enumerate}
These targets and their Fourier transforms are shown in~\Cref{fig:e2targets}.
\Cref{tab:kernel} shows training and test error for each target and activation
function.
To show the effect of different activation functions in Fourier space, we
reduce the Fourier transform of $f(\cdot)$ to the 1-dimensional function
\[ M(r) \triangleq \left|\hspace{3.3mm}\int\limits_{\smash{\mathclap{\|\vk\|_2=r}}} \F[f](\vk) \dd{\vk}\right| \]
which we call the \emph{total magnitude} at radius $r$.
\Cref{fig:total_mag} shows $M(r)$ for each target and each fit.

Most combinations of target and activation function are able to fit, with numerical innacuracy leading to occasional training
error values greater than $5\times10^{-8}$.
Stronger regularization (leading to steeper decay on total magnitude plots) yields lower test
performance in general, with some of the strongest regularizers (Erf, Sinc ($a=0.75$))
unable to fit.
Total magnitude plots mostly look similar, regardless of target function. One explanation is that the breakplane orientation does not affect the radial component, and the offset only changes
phase; so then deviations from $\F[\phi](k)$ can only be caused by interference
patterns. This explanation is plausible but needs further testing.

\subsection{Adaptive Regime}

\experiment \label{ex:adaptive_mnist} We trained shallow networks on MNIST
(input dimensionality $28 \times 28 = 784$, hidden layer of size $H = 200$,
output dimensionality $D_{out} = 10$) with SGD (10 epochs of 25 mini-batches,
256 training examples per mini-batch).
We vary the activation function across the set $\{$Step, ReLU, ReLU${}^2$,
ReLU${}^{3.5}$, Logistic, Cauchy, Sinc ($a=0.75$,$a=10$), Wavepacket
($\omega=2,\sigma=10$), Half-Exp, Oberhettinger
III.10~\cite{Oberhettinger1973Fourier}$\}$ (see~\Cref{tab:activation} for
definitions of these functions).
We show the training and test loss and accuracy
in~\Cref{fig:mnist_adaptive}~(a)--(d) which shows that many activation
functions are trainable, although some of the more unorthodox examples
(especially Sinc, Wavepacket) have very strong inductive biases that lead to
very poor generalization.
The Half-Exp and ReLU${}^{3.5}$ activations largely fail to train at all.
For Half-Exp, this seems to be due largely to the discontinuity at 0; the
Oberhettinger III.10~\cite{Oberhettinger1973Fourier} activation has similar
overall shape, but is continuous, and achieves reasonable training accuracy,
but low test accuracy.
For ReLU${}^{3.5}$, the driving cause is likely the extreme growth of the
active side, which causes small changes in parameters to cause huge changes in
the function far from the zero-plane, yielding instability.

\begin{figure*}
  \resizebox{\textwidth}{!}{\includegraphics{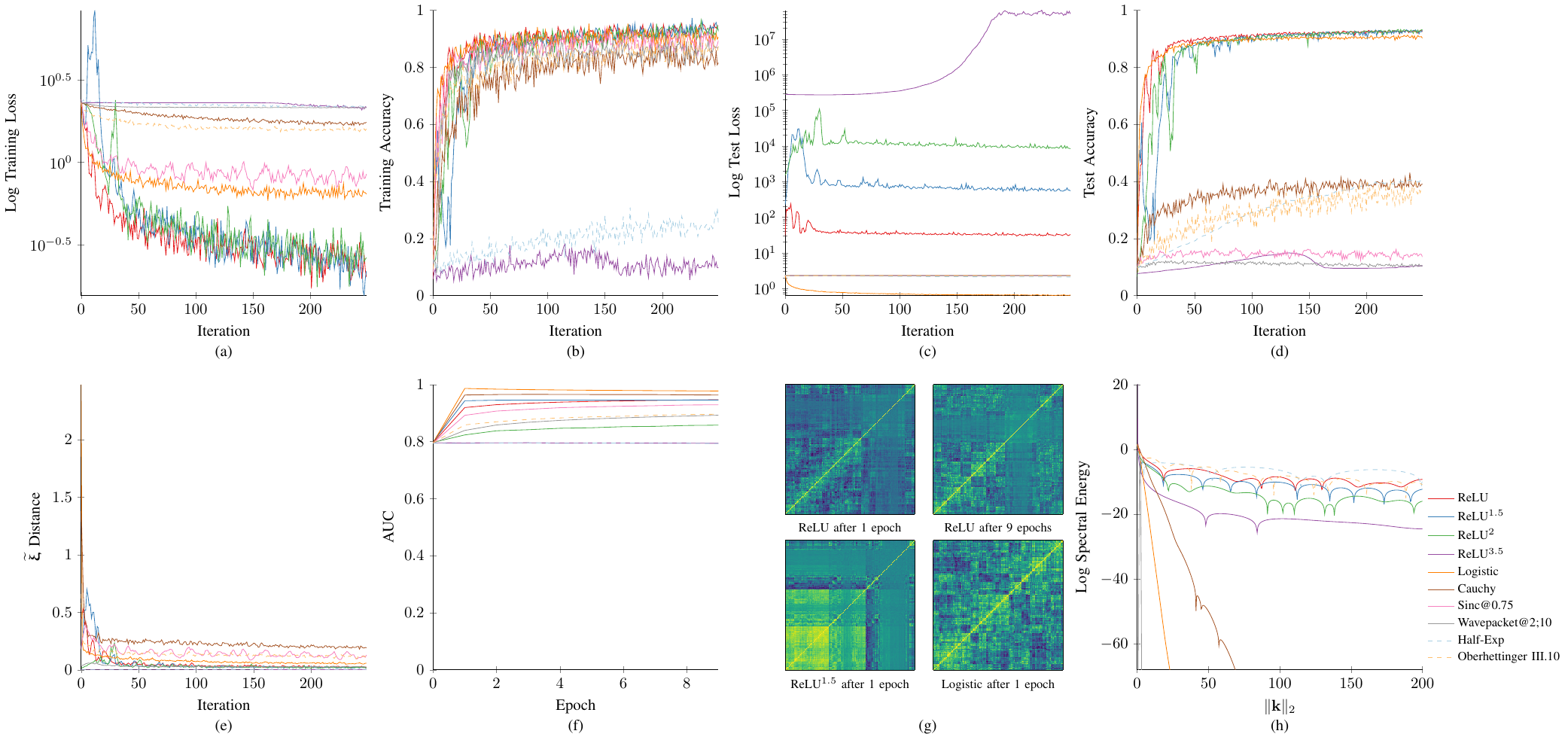}}
  \caption{Results for training networks with various activation functions to
  fit MNIST.
  \textbf{(a)--(d)}: Training loss, training accuracy, test loss, and test
  accuracy over training time.
  Loss plots on log scale.
  \textbf{(e)}: $\widetilde{\vxi}$ distance traveled per iteration.
  \textbf{(f)}: area under the curve of the cumulative sum of eigenvalues of the
  $\widetilde{\vxi}$ similarity matrix over time.
  \textbf{(g)}: 4 example $\widetilde{\vxi}$ similarity matrices.
  \textbf{(h)}: log total magnitude $\log E\left(\|\vk\|_2\right)$ for the
  final fits.
  The AUC plots and example matrices demonstrate that zero-plane clusters are
  formed rapidly.
  This is also supported by the $\widetilde{\vxi}$ distance plot, which shows
  higher values in early iterations.}
  \label{fig:mnist_adaptive}
\end{figure*}

We measure the extent of the zero-plane clustering phenomenon in two ways.
First, we plot per-training step $\widetilde{\vxi}$ distances, defined as the
average $\ell_2$ distance between matched pairs\footnote{A \emph{matched pair}
is any $\widetilde{\vxi}_i(t-1)$ and the closest $\widetilde{\vxi}_j(t)$;
typically, $i=j$, so that the pair is the same zero-plane at different
time-steps, but if breakplanes cross eachother, the ordering may change.
This definition is equivalent to the Wasserstein 2-distance between the
zero-plane distributions at times $t-1$ and $t$.} which shows a bias towards
early motion for many activation functions~\Cref{fig:mnist_adaptive}~(e).
Second, we calculate the matrix $S_ij \triangleq \langle\vxi_i,\vxi_j\rangle
\exp\left[-|\gamma_i-\gamma_j|^2\right]$ of
$\widetilde{\vxi}$-$\widetilde{\vxi}$ similarity scores, then plot the
cumulative sum of the sorted normalized eigenvalues of this matrix, and compute
the Area Under the Curve (AUC).
A high amount of clustering manifests as a left-skewed plot with an AUC near
1~\Cref{fig:mnist_adaptive}~(f).
We also show a few examples of $\vS$ in~\Cref{fig:mnist_adaptive}~(g), with
rows and columns sorted so that clusters of similar zero-planes are shown
together as lighter-colored blocks.

Additionally, we show the log total magnitude of the final trained networks
in~\Cref{fig:mnist_adaptive}~(h).
As in Experiment~\ref{ex:kernelconv}, the total magnitude at radius $r$ is not
changed much by adaptivity.
Additionally, with higher input dimension, breakplanes are sparser, making it
more difficult to orchestrate interference patterns.

\experiment Experiment~\ref{ex:adaptive_mnist} showed that after only a few
epochs, there is already a large amount of clustering, and zero-plane movement
is reduced, suggesting the possibility of switching to the (computationally
cheaper) kernel regime.
This is consistent with previous work \cite{Atanasov2022Neural}.

We trained the same networks as in Experiment~\ref{ex:adaptive_mnist} for 2
epochs, then switched to pure kernel training.
\Cref{fig:mnist_adaptive_then_kernel_gd}~(a)--(d) show training and test loss
over training time.
Note that many networks, including the best and some of the worst performers
(ReLU family, Logistic, Wavepacket) remain unchanged in terms of training loss,
but others (Cauchy, Sinc, Half-Exp, Oberhettinger) manage to achieve a better
training loss.
Training accuracy is less affected, except for a dramatic improvement seen for
the Half-Exp activation.
The test loss is also largely unchanged, except for a dramatic improvement for
ReLU${}^{3.5}$ and a minor improvement for Half-Exp.

\Cref{fig:mnist_adaptive_then_kernel_gd}~(e) shows log total magnitude, which
is qualitatively unchanged from Experiment~\ref{ex:adaptive_mnist}.

\section{Discussion}

Reparameterizing shallow neural networks in terms of the functional (Radon
spline) parameters enables a representation that is much more intuitive than
the original weight-space parameterization.
It directly enables novel results relating kernel regime implicit
regularization to Fourier regularization, and makes adaptive regime results
more interpretable.
This work has focused on the kernel regime, leaving the adaptive regime results
limited, but future work may be able to leverage this machinery to calculate
similar results in the adaptive regime.

Another potentially fruitful direction for future work is to compute closed
form expressions for kernel regime fits, which could lead to more efficient
optimization or direct calculations.
This could useful if the target function known to be smooth, as is the case in
e.g. the energy Hamiltonians in all-atom models of protein
folding~\cite{Yang2022Construction}.
Combined with the ``rational design'' of activation functions, this could lead
to efficient models with controllable smoothness and adaptivity, taking us
closer towards a vision of human-interpretable design of neural splines in high
dimensions.

\begin{figure*}
  \resizebox{\textwidth}{!}{\includegraphics{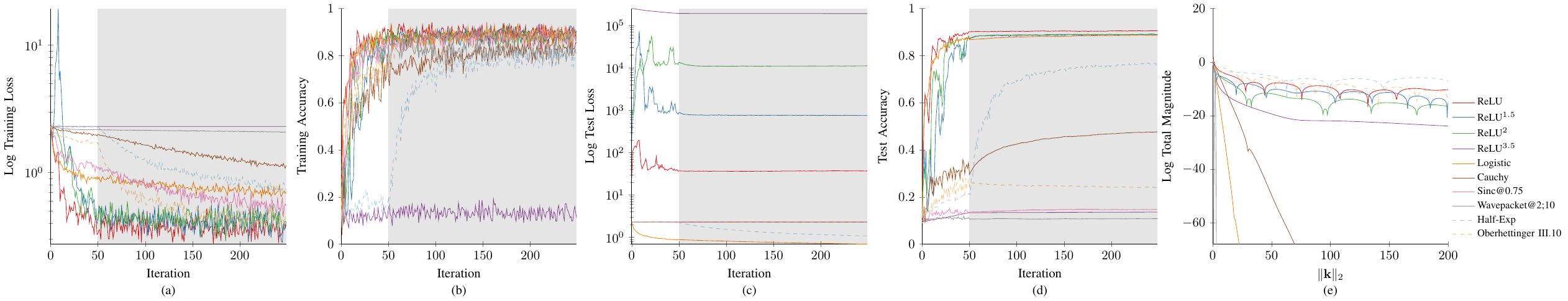}}
  \caption{Results for adaptive-then-kernel training networks with various
  activation functions to fit MNIST.
  \textbf{(a)--(d)}: Training loss, training accuracy, test loss, and test
  accuracy over training time.
  Loss plots on log scale; shaded region indicates kernel learning.
  \textbf{(h)}: log total magnitude $\log E\left(\|\vk\|_2\right)$ for the
  final fits.}
  \label{fig:mnist_adaptive_then_kernel_gd}
\end{figure*}

\bibliographystyle{IEEEtran}
\bibliography{main}

\clearpage
\onecolumn
\appendix
\section{}

\subsection{Radon Representation}

Starting with
\[ f_{\theta_{\tRS}(t)}
     = \R^*
       \left\{
         \left(
           \phi
           \ast_\gamma
           c_t(\vxi,\cdot)
           \eta_t(\vxi,\cdot)
         \right)
         (\gamma)
       \right\},
\]
we wish to solve for $c_t(\vxi,\gamma)$. Starting by inverting the $\R^*$:
\begin{IEEEeqnarray*}{r;l}
  (\phi \ast_\gamma c_t(\vxi,\cdot)\eta_t(\vxi,\cdot))(\gamma)
    &= -\kappa_D \R\left\{\left(-\nabla^2\right)^{(D-1)/2} f_{\theta_{\tRS}(t)} \right\}(\vxi,\gamma)\\
\shortintertext{Applying the Fourier transform with respect to $\gamma$ to both
sides yields}
  \left(\F_\gamma[\phi] \cdot \F_\gamma[c_t(\vxi,\cdot)\eta_t(\vxi,\cdot)]\right)(\vartheta)
    &= -\kappa_D \F_\gamma\left[\R\left\{\left(-\nabla^2\right)^{(D-1)/2} f_{\theta_{\tRS}(t)} \right\}(\vxi,\cdot)\right](\vartheta)\\
\shortintertext{Applying the Central Slice Theorem to the right hand side,}
    &= -\kappa_D \F_D\left[\left(-\nabla^2\right)^{(D-1)/2} f_{\theta_{\tRS}(t)}\right](\vartheta\vxi)\\
    &= -\kappa_D |\vartheta|^{D-1} \F_D\left[f_{\theta_{\tRS}(t)}\right](\vartheta\vxi)\\
\shortintertext{Then, dividing by $\F_\gamma[\phi](\vartheta)$}
  \F_\gamma\left[c_t(\vxi,\cdot)\eta_t(\vxi,\cdot)\right](\vartheta)
    &= -\kappa_D\frac{|\vartheta|^{D-1}}{\F_\gamma[\phi](\vartheta)}
       \F_D\left[f_{\theta_{\tRS}(t)}\right](\vartheta\vxi)\\
\shortintertext{Finally, apply the inverse Fourier transform with respect to
$\gamma$ to both sides, and divide by $\eta_t(\vxi,\gamma)$}
  c_t(\vxi,\gamma)
    &= \frac{-\kappa_D}{\eta_t(\vxi,\gamma)}
       \F_\gamma^{-1}
       \left[
         \frac{|\vartheta|^{D-1}}{\F_\gamma[\phi](\vartheta)}
         \F_D\left[f_{\theta_{\tRS}(t)}\right](\vartheta\vxi)
       \right](\gamma)\\
    &\triangleq
       \frac{1}{\eta_t(\vxi,\gamma)}
       \left(\R^*\right)^{-1}
       \left\{
         \L^{-1}_{\phi,\vxi} f_{\theta_{\tRS}(t)}
       \right\}
       (\vxi,\gamma),
\end{IEEEeqnarray*}
where $\L^{-1}_{\phi,\vxi}$ is the convolutional inverse of $\phi$, i.e.\ the
linear operator such that $\L^{-1}_{\phi,\vxi} \phi = \delta$, applied in the
direction of $\vxi$

\subsection{Kernel Regime Implicit Regularization}

Starting with
\[c_t(\vxi,\gamma)
    = \frac{1}{\eta_t(\vxi,\gamma)}
      \left(\R^*\right)^{-1}
      \left\{
        \L^{-1}_{\phi,\vxi} f_{\theta_{\tRS}(t)}
      \right\}
      (\vxi,\gamma),
\]
Then,
\begin{IEEEeqnarray*}{r;l}
  \|\vmu\|_2^2
    &= \sum_{i=1}^H \mu_i^2\\
\shortintertext{Applying the integral representation from \Cref{sec:radonintro}, we
get}
    &= \int\limits_{\mathclap{\SS^{D-1}\times\RR\times\RR}}
         \mu^2
       \dd{\eta_t(\vxi,\gamma,\mu)}\\
    &= \int\limits_{\mathclap{\SS^{D-1}\times\RR}}
         \hspace{1.3mm}
         \left(
           \hspace{0.3mm}
           \int\limits_\RR
             \mu^2
             \eta_0(\mu|\vxi,\gamma)
           \dd{\mu}
         \right)
       \dd{\eta_0(\vxi,\gamma)}\\
\shortintertext{Here, we have used the fact that in the kernel regime,
$\eta_t(\vxi,\gamma)$ does not change with $t$.
}
    &= \int\limits_{\mathclap{\SS^{D-1}\times\RR}}
         \EE_{\eta_0}\left[ \mu^2|\vxi,\gamma\right]
       \dd{\eta_0(\vxi,\gamma)}\\
    &= \int\limits_{\mathclap{\SS^{D-1}\times\RR}}
         \EE_{\eta_0}\left[\mu|\vxi,\gamma\right]^2
        +\operatorname{Var}_{\eta_0}\left[\mu|\vxi,\gamma\right]
       \dd{\eta_0(\vxi,\gamma)}\\
\shortintertext{In the setting of the optimization~\Cref{eq:kernelreg}, note that
the data-fitting term only depends on the mean $\int_\RR \mu
\eta_0(\mu|\vxi,\gamma)$, so we may always minimize the above integral by
setting $\Var[\mu|\vxi,\gamma] = 0$.
Thus, at the minimizer,
}
    &= \int\limits_{\mathclap{\SS^{D-1}\times\RR}}
         \EE_{\eta_0}\left[\mu|\vxi,\gamma\right]^2
       \dd{\eta_0(\vxi,\gamma)}\\
    &\triangleq
       \int\limits_{\mathclap{\SS^{D-1}\times\RR}}
         c_t(\vxi,\gamma)^2
       \dd{\eta_0}\!(\vxi,\gamma)\\
\shortintertext{Multiplying and dividing by $\eta_0(\vxi,\gamma)$, and explicitly
integrating over $\supp \eta_0$ to avoid dividing by 0, we have}
    &= \int\limits_{\mathclap{\supp \eta_0}}
         \frac{\left(c_t(\vxi,\gamma) \eta_0(\vxi,\gamma)\right)^2}{\eta_0(\vxi,\gamma)}
       \dd{\vxi}\dd{\gamma}\\
\shortintertext{Expanding $c_t(\cdot,\cdot)$ gives}
    &= \int\limits_{\mathclap{\supp \eta_0}}
         \frac
           {\left(
              \left(\R^*\right)^{-1}
              \left\{
                \L^{-1}_{\phi,\vxi} f_{\theta_{\tRS}(t)}
              \right\}(\vxi,\gamma)
            \right)^2
           }
           {\eta_0(\vxi,\gamma)}
       \dd{\vxi}\dd{\gamma}\\
\shortintertext{Taking $\argmin$ of both sides yields \Cref{eq:kernelreg}.}
\end{IEEEeqnarray*}

\subsection{Activation Functions}\label{sec:actderiv}

Here we derive the results summarized in \Cref{tab:activation}.

\paragraph{Power ReLU Family.}
Let $\lambda > 0$, and let $\DD_{+,\gamma}^\lambda$ be the right-sided
Riemann-Liouiville Fractional Derivative of order $\lambda$, w/r/t $\gamma$,
given by
\[ \DD_{+,\gamma}^\lambda f
   \triangleq
     \pdv[\lceil\lambda\rceil]{}{\gamma} \II_{+,\gamma}^{\lceil\lambda\rceil-\lambda} f,
\]
where $\II_{+,\gamma}^{\lambda}$ is the right-sided Riemann-Liouiville
Fractional Integral of order $\lambda$, w/r/t $\gamma$, given by
\[ \II_{+,\gamma}^\lambda f
   \triangleq
     \frac{1}{\Gamma(\lambda)}
     \int_{-\infty}^\gamma
       (\gamma-t)^{\lambda-1}
       f(t)
     \dd{t}.
\]
Then, if $f$ is well-behaved enough, specifically a Lizorkin function (a
Schwartz function whose Fourier transform is a Schwartz function $\psi$ such
that $\psi^{(k)}(0) = 0$ for all $k \ge 0$; equivalently, the Lizorkin space is
the space of Schwartz functions that are orthogonal to all polynomials), we
have
\[ \F_\gamma\left[\DD_{+,\gamma}^\lambda\right] = (i\vartheta)^\lambda, \qquad
   \F_\gamma\left[\II_{+,\gamma}^\lambda\right] = (i\vartheta)^{-\lambda}.
   \vspace{-2mm}
\]

Let
\begin{IEEEeqnarray*}{r;l}
  \left(\DD_{+,\vxi}^\lambda\, f\right)(\vx) &= \DD_{+,\gamma}^\lambda\, f(\vx + \gamma\vxi),\\
\shortintertext{%
i.e.\ take the fractional derivative of the 1-dimensional slice of $f(\cdot)$ at
$\vx$ in the direction of $\vxi$.
Then,}
  \F_D\left[\DD_{+,\vxi}^\lambda\, f\right](\vvartheta)
    &= \F_D\left[\DD_{+,\gamma}^\lambda\, f(\vx + \gamma\vxi)\right](\vvartheta)\\
    &= \int\limits_{\mathclap{\RR^D}}
         \DD_{+,\gamma}^\lambda\, f(\vx + \gamma\vxi)
         e^{-i\langle\vx,\vvartheta\rangle}
       \dd{\vx}\\
\shortintertext{Split $\RR^D$ into the parts parallel and perpendicular to $\vxi$:}
    &= \int\limits_{\mathclap{\RR^{D-1}\times\RR}}
           \DD_{+,x^\parallel}^\lambda\, f(\vx^\perp + x^\parallel\vxi)
           e^{-i\langle\vx^\perp + x^\parallel\vxi,\vvartheta\rangle}
         \dd{x^\parallel}
       \dd{\vx^\perp}\\
    &= \int\limits_{\mathclap{\RR^{D-1}\times\RR}}
           \DD_{+,x^\parallel}^\lambda\, f(\vx^\perp + x^\parallel\vxi)
           e^{-i\langle\vx^\perp + x^\parallel\vxi,\vvartheta^\perp + \vartheta^\parallel\vxi\rangle}
         \dd{x^\parallel}
       \dd{\vx^\perp}\\
    &= \int\limits_{\mathclap{\RR^{D-1}\times\RR}}
           \DD_{+,x^\parallel}^\lambda\, f(\vx^\perp + x^\parallel\vxi)
           e^{-i\left(\langle\vx^\perp,\vvartheta^\perp\rangle + x^\parallel\vartheta^\parallel\right)}
         \dd{x^\parallel}
       \dd{\vx^\perp}\\
    &= \int\limits_{\mathclap{\RR^{D-1}\times\RR}}
           \DD_{+,x^\parallel}^\lambda\, f(\vx^\perp + x^\parallel\vxi)
           e^{-i x^\parallel\vartheta^\parallel}
         \dd{x^\parallel}
         e^{-i\langle\vx^\perp,\vvartheta^\perp\rangle}
       \dd{\vx^\perp}\\
    &= \int\limits_{\mathclap{\RR^{D-1}}}
         \F_{x^\parallel}\left[ \DD_{+,x^\parallel}^\lambda\, f(\vx^\perp + x^\parallel\vxi)\right](\vartheta^\parallel)
         e^{-i\langle\vx^\perp,\vvartheta^\perp\rangle}
       \dd{\vx^\perp}\\
    &= \int\limits_{\mathclap{\RR^{D-1}}}
         (i\vartheta^\parallel)^\lambda \F_{x^\parallel}\left[ f(\vx^\perp + x^\parallel\vxi)\right](\vartheta^\parallel)
         e^{-i\langle\vx^\perp,\vvartheta^\perp\rangle}
       \dd{\vx^\perp}\\
    &= \F_{\vx^\perp}\left[
         (i\vartheta^\parallel)^\lambda \F_{x^\parallel}\left[ f(\vx^\perp + x^\parallel\vxi)\right](\vartheta^\parallel)
       \right](\vvartheta^\perp)\\
    &= (i\vartheta^\parallel)^\lambda
       \F_{\vx^\perp}\left[
         \F_{x^\parallel}\left[ f(\vx^\perp + x^\parallel\vxi)\right](\vartheta^\parallel)
       \right](\vvartheta^\perp)\\
    &= (i\vartheta^\parallel)^\lambda
       \F_D[f](\vvartheta)\\
    &= (i\langle\vvartheta,\vxi\rangle)^\lambda
       \F_D[f](\vvartheta)\\
\end{IEEEeqnarray*}
In the context of the Radon transform and the Central Slice Theorem, we have
$\vvartheta \triangleq \vartheta\vxi$; applying this, we get
\[ \F_D\left[\DD_{+,\vxi}^\lambda\, f\right](\vartheta\vxi)
   = (i\vartheta)^\lambda
     \F_D[f](\vartheta\vxi)
\]

Then, the activation function with this filter will be $\phi_\lambda(\cdot)$
such that $\F[\phi_\lambda]^{-1}(\vartheta) = (i\vartheta)^{-\lambda}$, i.e.\
convolution in $\gamma$ with this $\phi_\lambda(\cdot)$ is equivalent to the
Fractional integral $\II_{+,\gamma}^\lambda$, so that
\[ \phi_\lambda \ast f
   = \frac{1}{\Gamma(\lambda)}
     \int_{-\infty}^\gamma
       (\gamma-t)^{\lambda-1}
       f(t)
     \dd{t}.
\]
However, we can see that the right side this equation is already a convolution,
revealing that
\[ \phi_\lambda(z) = \frac{z^{\lambda-1}}{\Gamma(\lambda)} \Theta(z) = \frac{(z)_+^{\lambda-1}}{\Gamma(\lambda)} \]

Specializing for $\lambda=2$, we get $\phi_2(z) = (z)_+$ (i.e.\ the ReLU
activation), with corresponding filter $-k^2$ and regularizing operator
$\nabla^2$, as expected; for $\lambda=1$, we get $\phi_1(z) = \Theta(z)$ (i.e.\
the Heaviside distribution), with corresponding filter $ik$ and regularazing
operator $\partial_{\vxi}$.

\paragraph{SoftPlus Family.}
Consider the sigmoid function $\phi(x) = \operatorname{logistic}(\sigma x) =
\frac{e^{\sigma x}}{1+e^{\sigma x}}$.
First, note that
\begin{IEEEeqnarray*}{r;l}
  \phi'(x) &= \frac{\sigma}{4\cosh[2](\frac{\sigma\gamma}{2})}.\\
\shortintertext{Then, we will use the rule}
  \F\left[\int_{-\infty}^{x}f(\tau)\dd{\tau}\right](\vartheta)
    &= \pi\F[f](0)\delta(\vartheta) + \frac{\F[f](\vartheta)}{i\vartheta}\\
\shortintertext{%
Because we will always treat $\F_\gamma[\phi]$ as a Fourier multiplier against
a Lizorkin function, the $\delta(\vartheta)$ term can be treated as 0.
Applying this yields}
  \F_\gamma[\phi](\vartheta)
    &= \frac{1}{i\vartheta}\F_\gamma\left[\frac{\sigma}{4\cosh[2](\frac{\sigma\gamma}{2})}\right]\\
  \bnote{$(\gamma = \ln(t), t=u^{2/\sigma})$}
    &= \frac{2}{i\vartheta}
       \int_0^\infty
         \frac
           {u^{1-\frac{2i\vartheta}{\sigma}}}
           {(u^2 + 1)^2}
       \dd{u}\\
  \bnote{\small (GR 3.251.2; $\mu=2-\frac{2i\vartheta}{\sigma}, \nu=-1$)}
    &= \frac{1}{i\vartheta} B(1-\frac{i\vartheta}{\sigma}, 1+\frac{i\vartheta}{\sigma})\\
    &= \frac{\pi}{i\vartheta} \csc(\pi\frac{i\vartheta}{\sigma})\frac{1}{B(1,\frac{i\vartheta}{\sigma})}\\
    &= \frac{\pi}{i\vartheta} \csc(\pi\frac{i\vartheta}{\sigma})\frac{i\vartheta}{\sigma}\\
    &= \frac{1}{i\vartheta} \csch(\frac{\pi\vartheta}{\sigma})\frac{\pi\vartheta}{\sigma}\\
    &= - \frac{i\pi}{\sigma} \csch(\frac{\pi\vartheta}{\sigma})\\
\shortintertext{Thus, the filter associated with $\phi(x)$ is}
  \F_\gamma[\phi](\vartheta)^{-1}
    &= -\frac{\sigma}{i\pi}
       \sinh(\frac{\pi\vartheta}{\sigma})
\end{IEEEeqnarray*}
as expected.

The SoftPlus function is just the integral of the sigmoid, incurring a
$i\vartheta$ multiplier per the rule above.
Seeking to extend this family as with the Power ReLU family, we consider the
integral of the SoftPlus,
\[ \int_{-\infty}^z \ln(1+e^{\sigma x}) \dd{x}. \]
Performing the substitution $y \triangleq -e^{\sigma x}$
\[ \int_0^{-e^{\sigma z}} \frac{\ln(1-y)}{\sigma^2 y} \dd{y} \]
yields the result $-\frac{1}{\sigma^2} \operatorname{Li}_2\left(-e^{\sigma
z}\right)$, where $\operatorname{Li}_n(\cdot)$ is the polylogarithm of order
$n$; in particular,
\begin{IEEEeqnarray*}{r;l}
  \operatorname{Li}_0(z) &= \frac{z}{1-z}\\
  \operatorname{Li}_1(z) &= -\ln(1-z)\\
  \operatorname{Li}_n(z) &= \int_0^z \frac{\operatorname{Li}_{n-1}(t)}{t} \dd{t}
\end{IEEEeqnarray*}
Thus, additional integrals of the sigmoid/softplus will yield the order $n$
``Power SoftPlus'' $\phi_n = -\frac{1}{\sigma^n}
\operatorname{Li}_n\left(-e^{\sigma z}\right)$.
Empirically, taking the limit as the sharpness parameter approaches $\infty$
yields the Power ReLU of order $\lambda=n+1$, as expected.

\paragraph{SatReLU.}
Consider $\phi(x) = (x)_+ - (x-\Delta)_+$ (fixed-width un-normalized saturating relu)
\begin{IEEEeqnarray*}{r;l}
  \F_\gamma[\phi](\vartheta)
    &= -\frac{1}{\vartheta^2}\F[\delta(\gamma) - \delta(\gamma-\Delta)]\\
    &= -\frac{1}{\vartheta^2}\left(\F[\delta(\gamma)] - \F[\delta(\gamma-\Delta)]\right)\\
    &= -\frac{1}{\vartheta^2}\left(\F[\delta(\gamma)] - e^{-i\Delta\vartheta}\F[\delta(\gamma)]\right)\\
    &= -\frac{1}{\vartheta^2}\left(1 - e^{-i\Delta\vartheta}\right)\\
\end{IEEEeqnarray*}
$\displaystyle \lim_{\Delta\to\infty} e^{-i\Delta\vartheta}$ is undefined, but
takes on an ``average value'' of 0, which would yield the ReLU limit we
expect.

\paragraph{Wavepacket}
\begin{IEEEeqnarray*}{r;l}
  \F\left[
      \frac
        {\cos(\omega z)e^{-\frac{1}{2}\left(\frac{z}{\sigma}\right)^2}}
        {\sigma\sqrt{2\pi}}
    \right](k)
    &= \frac{1}{2\pi}
       \F\left[
         \cos(\omega z)
       \right](k)
       \ast
       \F\left[
         \frac
           {e^{-\frac{1}{2}\left(\frac{z}{\sigma}\right)^2}}
           {\sigma\sqrt{2\pi}}
       \right](k)\\
    &= \frac{1}{2\pi}
       \pi\left(\delta(k-\omega) + \delta(k+\omega)\right)
       \ast
       \exp(-\frac{\sigma^2k^2}{2})\\
    &= \frac
         {e^{-\frac{\sigma^2(k+\omega)^2}{2}} + e^{-\frac{\sigma^2(k-\omega)^2}{2}}}
         {2}
\end{IEEEeqnarray*}
as desired.

The remaining activation functions can be found in tables of typical Fourier
transform examples, or as the anti-derivative of such examples.

\subsection{Fourier transform of a finite-width network}
To calculate the ($D$-dimensional) Fourier transform of
$f_{\theta_{\tRS}}(\vx)$, we start by treating $f_{\theta_{\tRS}}(\vx)$ as a
distribution, which is defined in terms of its action on the test function
$\psi(\cdot)$:
\begin{IEEEeqnarray*}{r;l}
  \langle f_{\theta_{\tRS}}, \psi\rangle
    &= \sum_{j=1}^H
         \mu_i
         \langle \phi_{\omega_i}\mkern-2mu(\langle\vxi_i,\cdot\rangle-\gamma_i), \psi\rangle\\
    &= \sum_{j=1}^H
         \mu_i
         \int_{\RR^D}
           \phi_{\omega_i}\mkern-2mu(\langle\vxi_i,\vx\rangle-\gamma_i)
           \psi(\vx)
         \dd{\vx}\\
\shortintertext{Rotating $\vx$: let $R_{\vxi_i}$ be the (unitary) rotation matrix
that takes $\vxi_i$ to $\ve_1$; let $\tilde{\vx}$ denote the coordinates
rotated by $R_{\vxi_i}$}
    &= \sum_{j=1}^H
         \mu_i
         \int_{\RR^D}
           \phi_{\omega_i}\mkern-2mu(\langle\vxi_i,R'_{\vxi_i} \tilde{\vx}\rangle-\gamma_i)
           \psi(R'_{\vxi_i}\tilde{\vx})
         \dd{\tilde{\vx}}\\
    &= \sum_{j=1}^H
         \mu_i
         \int_{\RR^D}
           \phi_{\omega_i}\mkern-2mu(\langle R_{\vxi_i} \vxi_i, \tilde{\vx}\rangle-\gamma_i)
           \psi(R'_{\vxi_i}\tilde{\vx})
         \dd{\tilde{\vx}}\\
    &= \sum_{j=1}^H
         \mu_i
         \int_{\RR^D}
           \phi_{\omega_i}\mkern-2mu(\langle \ve_1, \tilde{\vx}\rangle-\gamma_i)
           \psi(R'_{\vxi_i}\tilde{\vx})
         \dd{\tilde{\vx}}\\
    &= \sum_{j=1}^H
         \mu_i
         \int_{\RR^D}
           \phi_{\omega_i}\mkern-2mu(\tilde{x}_1-\gamma_i)
           \psi(R'_{\vxi_i}\tilde{\vx})
         \dd{\tilde{\vx}}\\
\shortintertext{Then, the Fourier transform of a distribution is defined by the action
of the distribution on the Fourier transform of a test function:}
  \langle \F_D[f_{\theta_{\tRS}}], \psi\rangle
    &= \langle f_{\theta_{\tRS}}, \F_D[\psi]\rangle\\
    &= \sum_{j=1}^H
         \mu_i
         \int_{\RR^D}
           \phi_{\omega_i}\mkern-2mu(\tilde{x}_1-\gamma_i)
           \F_D[\psi](R'_{\vxi_i}\tilde{\vx})
         \dd{\tilde{\vx}}\\
    &= \sum_{j=1}^H
         \mu_i
         \int_{\RR^D}
           \phi_{\omega_i}\mkern-2mu(\tilde{x}_1-\gamma_i)
           \int_{\RR^D}
             \psi(\vz)
             e^{-i\langle R'_{\vxi_i}\tilde{\vx},\vz\rangle}
           \dd{\vz}
         \dd{\tilde{\vx}}\\
    &= \sum_{j=1}^H
         \mu_i
         \int_{\RR}
           \phi_{\omega_i}\mkern-2mu(\tilde{x}_1-\gamma_i)
           \int_{\RR^{D-1}}
             \int_{\RR^D}
               \psi(\vz)
               e^{-i\langle R'_{\vxi_i}\tilde{\vx},\vz\rangle}
             \dd{\vz}
           \dd{\tilde{\vx}_{2:D}}
         \dd{\tilde{x}_1}\\
    &= \sum_{j=1}^H
         \mu_i
         \int_{\RR}
           \phi_{\omega_i}\mkern-2mu(\tilde{x}_1-\gamma_i)
           \int_{\RR^D}
             \psi(\vz)
             \int_{\RR^{D-1}}
               e^{-i\langle R'_{\vxi_i}\tilde{\vx},\vz\rangle}
             \dd{\tilde{\vx}_{2:D}}
           \dd{\vz}
         \dd{\tilde{x}_1}\\
\shortintertext{Rotating both sides of a dot product by $R_{\vxi_i}$ leaves the dot
product invariant:}
    &= \sum_{j=1}^H
         \mu_i
         \int_{\RR}
           \phi_{\omega_i}\mkern-2mu(\tilde{x}_1-\gamma_i)
           \int_{\RR^D}
             \psi(\vz)
             \int_{\RR^{D-1}}
               e^{-i\langle \tilde{\vx},R_{\vxi_i}\vz\rangle}
             \dd{\tilde{\vx}_{2:D}}
           \dd{\vz}
         \dd{\tilde{x}_1}\\
\shortintertext{Reparam $\vz$ to $\tilde{\vz}$ via the same rotation $R_{\vxi_i}$:}
    &= \sum_{j=1}^H
         \mu_i
         \int_{\RR}
           \phi_{\omega_i}\mkern-2mu(\tilde{x}_1-\gamma_i)
           \int_{\RR^D}
             \psi(R'_{\vxi_i}\tilde{\vz})
             \int_{\RR^{D-1}}
               e^{-i\langle \tilde{\vx},\tilde{\vz}\rangle}
             \dd{\tilde{\vx}_{2:D}}
           \dd{\tilde{\vz}}
         \dd{\tilde{x}_1}\\
    &= \sum_{j=1}^H
         \mu_i
         \int_{\RR}
           \phi_{\omega_i}\mkern-2mu(\tilde{x}_1-\gamma_i)
           \int_{\RR^D}
             \psi(R'_{\vxi_i}\tilde{\vz})
             \int_\RR
               \cdots
               \int_\RR
                 e^{-i \tilde{x}_1 \tilde{z}_1}
                 e^{-i \tilde{x}_2 \tilde{z}_2}
                 \cdots
                 e^{-i \tilde{x}_D \tilde{z}_D}
               \dd{\tilde{x}_D}
             \cdots\dd{\tilde{x}_2}
           \dd{\tilde{\vz}}
         \dd{\tilde{x}_1}\\
    &= \sum_{j=1}^H
         \mu_i
         \int_{\RR}
           \phi_{\omega_i}\mkern-2mu(\tilde{x}_1-\gamma_i)
           \int_{\RR^D}
             \psi(R'_{\vxi_i}\tilde{\vz})
             e^{-i \tilde{x}_1 \tilde{z}_1}
             \prod_{d=2}^D
               \int_\RR
                 e^{-i \tilde{x}_d \tilde{z}_d}
               \dd{\tilde{x}_d}
           \dd{\tilde{\vz}}
         \dd{\tilde{x}_1}\\
    &= \sum_{j=1}^H
         \mu_i
         \int_{\RR}
           \phi_{\omega_i}\mkern-2mu(\tilde{x}_1-\gamma_i)
           \int_{\RR^D}
             \psi(R'_{\vxi_i}\tilde{\vz})
             e^{-i \tilde{x}_1 \tilde{z}_1}
             \prod_{d=2}^D
               \delta(\tilde{z}_d)
           \dd{\tilde{\vz}}
         \dd{\tilde{x}_1}\\
    &= \sum_{j=1}^H
         \mu_i
         \int_{\RR}
           \phi_{\omega_i}\mkern-2mu(\tilde{x}_1-\gamma_i)
           \int_{\RR^D}
             \psi(\vz)
             e^{-i \tilde{x}_1 (R_{\vxi_i}\vz)_1}
             \prod_{d=2}^D
               \delta((R_{\vxi_i}\vz)_d)
           \dd{\vz}
         \dd{\tilde{x}_1}\\
\shortintertext{The product-of-Diracs term selects the $\vz$ that are parallel to
$\vxi_i$, leaving only a 1-dimensional integral:}
    &= \sum_{j=1}^H
         \mu_i
         \int_{\RR}
           \phi_{\omega_i}\mkern-2mu(\tilde{x}_1-\gamma_i)
           \int_{\RR}
             \psi(z_1 \vxi_i)
             e^{-i \tilde{x}_1 (R_{\vxi_i} z_1 \vxi_i)_1}
           \dd{z_1}
         \dd{\tilde{x}_1}\\
    &= \sum_{j=1}^H
         \mu_i
         \int_{\RR}
           \phi_{\omega_i}\mkern-2mu(\tilde{x}_1-\gamma_i)
           \int_{\RR}
             \psi(z_1 \vxi_i)
             e^{-i \tilde{x}_1 z_1}
           \dd{z_1}
         \dd{\tilde{x}_1}\\
    &= \sum_{j=1}^H
         \mu_i
         \int_{\RR}
           \phi_{\omega_i}\mkern-2mu(\tilde{x}_1-\gamma_i)
           \F_1[\psi(\cdot\vxi_i)](\tilde{x}_1)
         \dd{\tilde{x}_1}\\
\shortintertext{Using $u$ as the 1-dimensional Fourier variable:}
    &= \sum_{j=1}^H
         \mu_i
         \left\langle
           \phi_{\omega_i}\mkern-2mu(u-\gamma_i)
         , \F_1[\psi(\cdot\vxi_i)](u)
         \right\rangle_1\\
    &= \sum_{j=1}^H
         \mu_i
         \left\langle
           \F_1[\phi_{\omega_i}\mkern-2mu(\cdot-\gamma_i)]
         , \psi(u\vxi_i)
         \right\rangle_1\\
    &= \sum_{j=1}^H
         \mu_i
         \left\langle
           e^{-i\gamma_i u}\F_1[\phi_{\omega_i}](u)
         , \psi(u\vxi_i)
         \right\rangle_1\\
\shortintertext{Thus, $\langle\F_D[f_{\theta_{\tRS}}],\psi\rangle$ computes the Fourier transform
of the activation and integrates it against $\psi$ along lines parallel to
each $\vxi_i$.
That is, $\F_D[f_{\theta_{\tRS}}]$ is a distribution that consists of a sum of ``weighted
Dirac-lines'' with weight $\mu_i e^{-i\gamma_i u}\F_1[\phi_{\omega_i}](u)$ along the
line $\{u\vxi_i|u\in\RR\}\equiv \RR\vxi_i$.
If we define the distribution $\delta_{\vxi_i}\!(\vk)$ by}
  \langle \delta_{\vxi_i},\psi\rangle
    &\triangleq
       \int_\RR
         \psi(u\vxi_i)
       \dd{u}\\
\shortintertext{Then, for any smooth $g(\vx)$}
  \langle g \delta_{\vxi_i},\psi\rangle
    &\triangleq
       \langle \delta_{\vxi_i},g\psi\rangle
     = \int_\RR
         g(u\vxi_i) \psi(u\vxi_i)
       \dd{u}\\
\shortintertext{Using $g(\vx)\triangleq e^{-i\gamma_i
\langle\vk,\vxi_i\rangle}\F_1[\phi](\langle\vk,\vxi_i\rangle)$, we have}
  \F_D[f_{\theta_{\tRS}}](\vk)
    &= \sum_{j=1}^H
         \mu_i
         e^{-i\gamma_i \langle\vk,\vxi_i\rangle}
         \F_1[\phi_{\omega_i}](\langle\vk,\vxi_i\rangle)
         \delta_{\vxi_i}\!(\vk)\\
\end{IEEEeqnarray*}

\subsection{Fourier Interpretation}

Here we collect proofs of the equations of~\Cref{sec:fourier,sec:curse}.

\begin{lemma}
\[
  \int_{\SS^{D-1}\times\RR} \left(\F^{-1}_\gamma\left[\F_D[f](\vartheta\vxi)\right](\gamma)\right)^2 \dd{\vxi} \dd{\gamma}
  = 2\int_{\RR^D} \frac{1}{k^{D-1}} \left|\F_D[f](\vk)\right|^2 \dd{\vk}
\]
\end{lemma}

\begin{proof}
\begin{IEEEeqnarray*}{l}
  \phantom{=}
    \int_{\SS^{D-1}\times\RR} \left(\F^{-1}_\gamma\left[\F_D[f](\vartheta\vxi)\right](\gamma)\right)^2 \dd{\vxi} \dd{\gamma}\\
  = \int_{\SS^{D-1}\times\RR} \left|\F_D[f](\vartheta\vxi)\right|^2 \dd{\vxi} \dd{\vartheta}\\
  = \int_{\SS^{D-1}} \int_0^\infty \left|\F_D[f](\vartheta\vxi)\right|^2 \dd{\vartheta} \dd{\vxi} + \int_{\SS^{D-1}} \int_0^\infty \left|\F_D[f](-\vartheta\vxi)\right|^2 \dd{\vartheta} \dd{\vxi}\\
  = \int_{\SS^{D-1}} \int_0^\infty \left|\F_D[f](\vartheta\vxi)\right|^2 \dd{\vartheta} \dd{\vxi} + \int_{\SS^{D-1}} \int_0^\infty \left|\F_D[f](\vartheta(-\vxi))\right|^2 \dd{\vartheta} \dd{\vxi}\\
  = \int_{\SS^{D-1}} \int_0^\infty \left|\F_D[f](\vartheta\vxi)\right|^2 \dd{\vartheta} \dd{\vxi} + \int_{\SS^{D-1}} \int_0^\infty \left|\F_D[f](\vartheta\vxi)\right|^2 \dd{\vartheta} \dd{\vxi}\\
  = 2\int_{\SS^{D-1}} \int_0^\infty \left|\F_D[f](\vartheta\vxi)\right|^2 \dd{\vartheta} \dd{\vxi}\\
  = 2\int_{\RR^D} \frac{1}{k^{D-1}} \left|\F_D[f](\vk)\right|^2 \dd{\vk}
\end{IEEEeqnarray*}
\end{proof}

\begin{lemma}
\[
  \int_{\SS^{D-1}\times\RR} \left(\F^{-1}_\gamma\left[|\vartheta|^{D-1}\F_D[f](\vartheta\vxi)\right](\gamma)\right)^2 \dd{\vxi} \dd{\gamma}
  = 2\int_{\RR^D} \left|k^{(D-1)/2} \F_D[f](\vk)\right|^2 \dd{\vk}
\]
\end{lemma}

\begin{proof}
\begin{IEEEeqnarray*}{l}
  \phantom{=}
    \int_{\SS^{D-1}\times\RR} \left(\F^{-1}_\gamma\left[|\vartheta|^{D-1}\F_D[f](\vartheta\vxi)\right](\gamma)\right)^2 \dd{\vxi} \dd{\gamma}\\
  = \int_{\SS^{D-1}\times\RR} \left||\vartheta|^{D-1}\F_D[f](\vartheta\vxi)\right|^2 \dd{\vxi} \dd{\vartheta}\\
  = \int_{\SS^{D-1}} \int_0^\infty \left||\vartheta|^{D-1}\F_D[f](\vartheta\vxi)\right|^2 \dd{\vartheta} \dd{\vxi} + \int_{\SS^{D-1}} \int_0^\infty \left||{-\vartheta}|^{D-1}\F_D[f](-\vartheta\vxi)\right|^2 \dd{\vartheta} \dd{\vxi}\\
  = \int_{\SS^{D-1}} \int_0^\infty \left||\vartheta|^{D-1}\F_D[f](\vartheta\vxi)\right|^2 \dd{\vartheta} \dd{\vxi} + \int_{\SS^{D-1}} \int_0^\infty \left||\vartheta|^{D-1}\F_D[f](\vartheta(-\vxi))\right|^2 \dd{\vartheta} \dd{\vxi}\\
  = \int_{\SS^{D-1}} \int_0^\infty \left||\vartheta|^{D-1}\F_D[f](\vartheta\vxi)\right|^2 \dd{\vartheta} \dd{\vxi} + \int_{\SS^{D-1}} \int_0^\infty \left||\vartheta|^{D-1}\F_D[f](\vartheta\vxi)\right|^2 \dd{\vartheta} \dd{\vxi}\\
  = 2\int_{\SS^{D-1}} \int_0^\infty \left||\vartheta|^{D-1}\F_D[f](\vartheta\vxi)\right|^2 \dd{\vartheta} \dd{\vxi}\\
  = 2\int_{\RR^D} \frac{1}{k^{D-1}} \left|k^{D-1} \F_D[f](\vk)\right|^2 \dd{\vk}\\
  = 2\int_{\RR^D} \left|k^{(D-1)/2} \F_D[f](\vk)\right|^2 \dd{\vk}
\end{IEEEeqnarray*}

\end{proof}

\begin{lemma}
\[
  \int_{\SS^{D-1}\times\RR} \left(\F^{-1}_\gamma\left[\frac{|\vartheta|^{D-1}}{\F_\gamma[\phi](\vartheta)}\F_D[f](\vartheta\vxi)\right](\gamma)\right)^2 \dd{\vxi} \dd{\gamma}
  = 2\int_{\RR^D} \left|\frac{k^{(D-1)/2}}{\F[\phi](k)} \F_D[f](\vk)\right|^2 \dd{\vk}
\]
\end{lemma}

\begin{proof}
\begin{IEEEeqnarray*}{l}
  \phantom{=}
    \int_{\SS^{D-1}\times\RR} \left(\F^{-1}_\gamma\left[\frac{|\vartheta|^{D-1}}{\F_\gamma[\phi](\vartheta)}\F_D[f](\vartheta\vxi)\right](\gamma)\right)^2 \dd{\vxi} \dd{\gamma}\\
  = \int_{\SS^{D-1}\times\RR} \left|\frac{|\vartheta|^{D-1}}{\F_\gamma[\phi](\vartheta)}\F_D[f](\vartheta\vxi)\right|^2 \dd{\vxi} \dd{\vartheta}\\
  = \int_{\SS^{D-1}} \int_0^\infty \left|\frac{|\vartheta|^{D-1}}{\F_\gamma[\phi](\vartheta)}\F_D[f](\vartheta\vxi)\right|^2 \dd{\vartheta} \dd{\vxi} + \int_{\SS^{D-1}} \int_0^\infty \left|\frac{|{-\vartheta}|^{D-1}}{\F_\gamma[\phi](-\vartheta)}\F_D[f](-\vartheta\vxi)\right|^2 \dd{\vartheta} \dd{\vxi}\\
  = \int_{\SS^{D-1}} \int_0^\infty \left|\frac{|\vartheta|^{D-1}}{\F_\gamma[\phi](\vartheta)}\F_D[f](\vartheta\vxi)\right|^2 \dd{\vartheta} \dd{\vxi} + \int_{\SS^{D-1}} \int_0^\infty \left|\frac{|\vartheta|^{D-1}}{\F_\gamma[\phi](-\vartheta)}\F_D[f](\vartheta(-\vxi))\right|^2 \dd{\vartheta} \dd{\vxi}\\
  = \int_{\SS^{D-1}} \int_0^\infty \left|\frac{|\vartheta|^{D-1}}{\F_\gamma[\phi](\vartheta)}\F_D[f](\vartheta\vxi)\right|^2 \dd{\vartheta} \dd{\vxi} + \int_{\SS^{D-1}} \int_0^\infty \left|\frac{|\vartheta|^{D-1}}{\F_\gamma[\phi](-\vartheta)}\F_D[f](\vartheta\vxi)\right|^2 \dd{\vartheta} \dd{\vxi}\\
  = \int_{\SS^{D-1}} \int_0^\infty \left|\frac{|\vartheta|^{D-1}}{\F_\gamma[\phi](\vartheta)}\F_D[f](\vartheta\vxi)\right|^2 \dd{\vartheta} \dd{\vxi} + \int_{\SS^{D-1}} \int_0^\infty \left|\frac{|\vartheta|^{D-1}}{(\F_\gamma[\phi](\vartheta))^*}\F_D[f](\vartheta\vxi)\right|^2 \dd{\vartheta} \dd{\vxi}\\
\shortintertext{Conjugation inside the modulus has no effect:}
  = \int_{\SS^{D-1}} \int_0^\infty \left|\frac{|\vartheta|^{D-1}}{\F_\gamma[\phi](\vartheta)}\F_D[f](\vartheta\vxi)\right|^2 \dd{\vartheta} \dd{\vxi} + \int_{\SS^{D-1}} \int_0^\infty \left|\frac{|\vartheta|^{D-1}}{\F_\gamma[\phi](\vartheta)}\F_D[f](\vartheta\vxi)\right|^2 \dd{\vartheta} \dd{\vxi}\\
  = 2\int_{\SS^{D-1}} \int_0^\infty \left|\frac{|\vartheta|^{D-1}}{\F_\gamma[\phi](\vartheta)}\F_D[f](\vartheta\vxi)\right|^2 \dd{\vartheta} \dd{\vxi}\\
  = 2\int_{\RR^D} \left|\frac{k^{(D-1)/2}}{\F[\phi](k)} \F_D[f](\vk)\right|^2 \dd{\vk}\\
\end{IEEEeqnarray*}
\end{proof}

\begin{lemma}
\[
  \|f_\varepsilon(\vx)\|_{\R,\phi,\eta_0}^2
     = \varepsilon^{-1}
       \int\limits_{\mathclap{\SS^{D-1}\times\RR}}
          \frac{\kappa_D^2}{\eta_0(\vxi,\varepsilon\gamma)}
          \left(
            \F_\gamma^{-1}
            \left[
              \frac{|\vartheta|^{D-1}}{\F_\gamma[\phi_\varepsilon](\vartheta)}
              \F_D[f](\vartheta\vxi)
            \right](\gamma)
          \right)^2
        \dd{\vxi}\dd{\gamma}
\]
\end{lemma}

\begin{proof}
\begin{IEEEeqnarray*}{r;l}
  \|f_\varepsilon(\vx)\|_{\R,\phi,\eta_0}^2
    &= \int_{\SS^{D-1}\times\RR}
         \frac{1}{\eta_0(\vxi,\gamma)}
         \left(
           \F_\gamma^{-1}\left[
             \frac{|\vartheta|^{D-1}}{\F_\gamma[\phi](\vartheta)}
             \F_D\left[f_\varepsilon\right](\vartheta\vxi)
           \right](\gamma)
         \right)^2
       \dd{\vxi}\dd{\gamma}\\
    &= \int_{\SS^{D-1}\times\RR}
         \frac{1}{\eta_0(\vxi,\gamma)}
         \left(
           \F_\gamma^{-1}\left[
             \frac{|\vartheta|^{D-1}}{\F_\gamma[\phi](\vartheta)}
             \F_D[f(\cdot/\varepsilon)](\vartheta\vxi)
           \right](\gamma)
         \right)^2
       \dd{\vxi}\dd{\gamma}\\
    &= \int_{\SS^{D-1}\times\RR}
         \frac{1}{\eta_0(\vxi,\gamma)}
         \left(
           \F_\gamma^{-1}\left[
             \frac{|\vartheta|^{D-1}}{\F_\gamma[\phi](\vartheta)}
             \varepsilon^D \F_D[f](\varepsilon\vartheta\vxi)
           \right](\gamma)
         \right)^2
       \dd{\vxi}\dd{\gamma}\\
    &= \varepsilon^{2D}
       \int_{\SS^{D-1}\times\RR}
         \frac{1}{\eta_0(\vxi,\gamma)}
         \left(
           \F_\gamma^{-1}\left[
             \frac{|\vartheta|^{D-1}}{\F_\gamma[\phi](\vartheta)}
             \F_D[f](\varepsilon\vartheta\vxi)
           \right](\gamma)
         \right)^2
       \dd{\vxi}\dd{\gamma}\\
    &= \varepsilon^{2D}
       \int_{\SS^{D-1}\times\RR}
         \frac{1}{\eta_0(\vxi,\gamma)}
         \left(
           \varepsilon^{-1}
           \F_\gamma^{-1}\left[
             \frac{|\vartheta/\varepsilon|^{D-1}}{\F_\gamma[\phi](\vartheta/\varepsilon)}
             \F_D[f](\vartheta\vxi)
           \right](\gamma/\varepsilon)
         \right)^2
       \dd{\vxi}\dd{\gamma}\\
    &= \varepsilon^{2D-2}
       \int_{\SS^{D-1}\times\RR}
         \frac{1}{\eta_0(\vxi,\gamma)}
         \left(
           \F_\gamma^{-1}\left[
             \frac{|\vartheta/\varepsilon|^{D-1}}{\F_\gamma[\phi](\vartheta/\varepsilon)}
             \F_D[f](\vartheta\vxi)
           \right](\gamma/\varepsilon)
         \right)^2
       \dd{\vxi}\dd{\gamma}\\
    &= \varepsilon^{2D-1}
       \int_{\SS^{D-1}\times\RR}
         \frac{1}{\eta_0(\vxi,\varepsilon\gamma)}
         \left(
           \F_\gamma^{-1}\left[
             \frac{|\vartheta/\varepsilon|^{D-1}}{\F_\gamma[\phi](\vartheta/\varepsilon)}
             \F_D[f](\vartheta\vxi)
           \right](\gamma)
         \right)^2
       \dd{\vxi}\dd{\gamma}\\
    &= \varepsilon^{2D-1}
       \int_{\SS^{D-1}\times\RR}
         \frac{1}{\eta_0(\vxi,\varepsilon\gamma)}
         \left(
           \frac{1}{\varepsilon^{D-1}}
           \F_\gamma^{-1}\left[
             \frac{|\vartheta|^{D-1}}{\F_\gamma[\phi](\vartheta/\varepsilon)}
             \F_D[f](\vartheta\vxi)
           \right](\gamma)
         \right)^2
       \dd{\vxi}\dd{\gamma}\\
    &= \varepsilon^{2D-1 - 2(D-1)}
       \int_{\SS^{D-1}\times\RR}
         \frac{1}{\eta_0(\vxi,\varepsilon\gamma)}
         \left(
           \F_\gamma^{-1}\left[
             \frac{|\vartheta|^{D-1}}{\F_\gamma[\phi](\vartheta/\varepsilon)}
             \F_D[f](\vartheta\vxi)
           \right](\gamma)
         \right)^2
       \dd{\vxi}\dd{\gamma}\\
    &= \varepsilon
       \int_{\SS^{D-1}\times\RR}
         \frac{1}{\eta_0(\vxi,\varepsilon\gamma)}
         \left(
           \F_\gamma^{-1}\left[
             \frac{|\vartheta|^{D-1}}{\F_\gamma[\phi](\vartheta/\varepsilon)}
             \F_D[f](\vartheta\vxi)
           \right](\gamma)
         \right)^2
       \dd{\vxi}\dd{\gamma}\\
    &= \varepsilon
       \int_{\SS^{D-1}\times\RR}
         \frac{1}{\eta_0(\vxi,\varepsilon\gamma)}
         \left(
           \F_\gamma^{-1}\left[
             \frac{|\vartheta|^{D-1}}{\varepsilon\F_\gamma[\phi(\cdot\varepsilon)](\vartheta)}
             \F_D[f](\vartheta\vxi)
           \right](\gamma)
         \right)^2
       \dd{\vxi}\dd{\gamma}\\
    &= \varepsilon^{-1}
       \int_{\SS^{D-1}\times\RR}
         \frac{1}{\eta_0(\vxi,\varepsilon\gamma)}
         \left(
           \F_\gamma^{-1}\left[
             \frac{|\vartheta|^{D-1}}{\F_\gamma[\phi(.\varepsilon)](\vartheta)}
             \F_D[f](\vartheta\vxi)
           \right](\gamma)
         \right)^2
       \dd{\vxi}\dd{\gamma}
\end{IEEEeqnarray*}
\end{proof}
\end{document}